\documentclass{article}

\usepackage[final]{neurips_2022}

\usepackage[utf8]{inputenc} 
\usepackage[T1]{fontenc}    
\usepackage{url}            
\usepackage{booktabs}       
\usepackage{amsfonts}       
\usepackage{nicefrac}       
\usepackage{microtype}      
\usepackage{xcolor}         
\usepackage{graphicx}
\usepackage{float}
\usepackage{amsmath}
\usepackage{amssymb}
\usepackage{mathtools}
\usepackage{amsthm}
\usepackage{multirow}

\usepackage{thmtools, thm-restate}
\usepackage{microtype}
\usepackage{graphicx}
\usepackage{booktabs} 
\usepackage{amsmath}

\DeclareMathOperator*{\argmax}{arg\,max}
\DeclareMathOperator*{\EX}{\mathbb{E}}

\usepackage{algorithm}
\usepackage{algpseudocode}
\usepackage{bbm}
\usepackage{subcaption}
\usepackage{caption}
\usepackage{wrapfig}
\usepackage{dsfont}
\algtext*{Indent}
\algtext*{EndIndent}

\usepackage[multiple]{footmisc}

\theoremstyle{plain}

\newtheorem{lemma}{Lemma}

\theoremstyle{definition}
\newtheorem{definition}{Definition}

\theoremstyle{remark}

\definecolor{mydarkblue}{rgb}{0,0.08,0.45}
\usepackage[colorlinks=true,linkcolor=mydarkblue,citecolor=mydarkblue]{hyperref}

\title{Distributed Influence-Augmented Local Simulators for Parallel MARL in Large Networked Systems}

\author{
  Miguel Suau\\
  Delft University of Technology \\
  m.suaudecastro@tudelft.nl
  \And
  Jinke He \\
  Delft University of Technology \\
  j.he-4@tudelft.nl
  \And
  Mustafa Mert \c{C}elikok\\
  Aalto University \\
  mustafa.celikok@aalto.fi
  \And
  Matthijs T. J. Spaan \\
  Delft University of Technology \\
  m.t.j.spaan@tudelft.nl
  \And
  Frans A. Oliehoek \\
  Delft University of Technology \\
  f.a.oliehoek@tudelft.nl
}

\begin{document}

\maketitle
\begin{abstract}

Due to its high sample complexity, simulation is, as of today, critical for the successful application of reinforcement learning. Many real-world problems, however, exhibit overly complex dynamics, making their full-scale simulation computationally slow. In this paper, we show how to factorize large networked systems of many agents into multiple local regions such that we can build separate simulators that run independently and in parallel. To monitor the influence that the different local regions exert on one another, each of these simulators is equipped with a learned model that is periodically trained on real trajectories.
Our empirical results reveal that distributing the simulation among different processes not only makes it possible to train large multi-agent systems in just a few hours but also helps mitigate the negative effects of simultaneous learning.\footnote{Source code is available at \url{https://github.com/INFLUENCEorg/DIALS}.}
\end{abstract}

\section{Introduction}



Imagine we have to train a team of agents to control the traffic lights of a very large city, so large that we simply cannot control all traffic lights using a single policy. The first step would be to split the problem into multiple sub-regions. A natural division would be to assign one traffic light to each agent. Then, since the agents act locally, we would limit their observations to contain only local information. This partial observability could affect their optimal policies but would also make each individual decision-making problem more manageable \citep{McCallum95PhD, Dearden97AIJ}. Moreover, we may also want to reward agents only for what occurs in their local neighborhood such that we reduce the variance of the returns \citep{spooner2021factored} and facilitate credit assignment \citep{castellini2020difference}. Finally, we could train all agents together on a big traffic simulator that reproduces the global dynamics. However, if the city is truly large, it could take weeks or even months to optimize their policies. That is assuming training actually converges. 

One may argue that, since the agents' observations and rewards are local, we could as well train them on separate simulators that model only the local transition dynamics (i.e. cars moving within each of the sub-regions; \citealt{VanDerPol16NIPSWS}). This approach might work if the agents' local transitions are isolated from the rest of the system \citep{Becker03AAMAS}, but would probably break when the local regions are coupled. This is because the local simulators would fail to account for the fact that the agents' local regions belong to a larger system and depend on one another. 
A solution is to model the influence the global system exerts on each local region. Fortunately, this does not necessarily imply modelling the entire system, or else we would just use the global simulator. In many scenarios, such as in the traffic problem, even though the local regions may be affected by many external variables (e.g.traffic densities in other parts of the city), they are only directly influenced by a small subset of them (e.g., road segments that connect the intersections with the rest of the city). This subset of variables is known as the influence sources. The theoretical framework of Influence-Based Abstraction \citep{oliehoek2021sufficient} shows that by monitoring the posterior distribution of the influence sources given the action local state history (ALSH), one can simulate realistic trajectories that match those produced by the global simulator. The resulting simulator, known as the influence-augmented local simulator (IALS), has been proven effective in single agent scenarios when combined with planning \citep{he2020influence} and reinforcement learning (RL) algorithms \citep{suau2022influence}.

In this paper, we extend the IBA framework to multi-agent domains. We show how to factorize large networked systems, such as the previous traffic example, into multiple sub-regions so that we can replace the GS by a distributed network of IALSs that can run independently and in parallel. There is one important caveat to this. The IBA framework assumes only a single agent is learning at a time. This assumption is needed to make the influence distributions stationary. This implies that in our case since we want the agents to learn simultaneously, previously computed influence distributions
would no longer be valid after the agents update their policies. The naive solution would be to recompute new influence distributions every time any agent updates its policy. However, we argue that this is not only impractical, since recomputing the distributions is not without costs, but also undesirable. The theoretical results in Section \ref{sec:parallelization} demonstrate that multiple (similar) joint policies may induce the same influence distributions and that even when they vary a little, they can still elicit the same optimal policies. Further, our insights in Section \ref{sec:mitigating} hint that what seems to be a problem at first, may in fact be an advantage since in many situations, maintaining the previous influence distributions implies that the local transitions, although biased, remain stationary. 
\paragraph{Contributions} The main contributions of this paper are: (1) adapting IBA to multi-agent reinforcement learning (MARL),\footnote{Although the original IBA formulation \citep{oliehoek2012influence} is already framed as multi-agent, it assumes agents learn one at a time while the other agents' policies are fixed.} and demonstrating that simultaneous learning is possible without incurring major computational costs,
(2) showing that by distributing the simulation among different processes, we can parallelize training and scale up to systems with many agents, (3) revealing that the non-stationarity issues inherent to MARL are partly mitigated as a result of this training scheme.
\section{Related Work}
A few prior works have investigated the computational benefits of factorizing large systems into independent local regions \citep{Nair05AAAI,Varakantham07AAMAS, Kumar11IJCAI, Witwicki11AAMAS}. Unfortunately, since local regions are often coupled to one another, such factorizations are not always appropriate. Nonetheless, in many cases, the interactions between regions occur through a limited number of variables. Using this property, the theoretical work by \citet{oliehoek2021sufficient} on influence-based abstraction (IBA) describes how to build influence-augmented local simulators (IALS) of local-POMDPs, which model only the variables in the environment that are directly relevant to the agent while monitoring the response of the rest of the system with the influence predictor. The problem is that the exact computation of the conditional influence distribution is intractable, and we can only try to estimate it from data. \citet{Congeduti21AAMAS} provide theoretical bounds on the value loss when planning with approximate influence predictors. The work by \citet{he2020influence} has empirically demonstrated the advantage of this approach to improve the efficiency of online planning in two discrete toy problems.  \citet{suau2022influence} scale the method to high-dimensional problems by integrating the IBA framework with single-agent RL showing that the IALS can train policies much faster than the GS. In this paper, we extend the IBA solution to MARL and explain how to build a network of independent IALS such that we can train agents in parallel.

One of the consequences of training agents on independent simulators is that the non-stationarity issues arising from having the agents learn simultaneously are partly mitigated. There is a sizeable body of literature that concentrates on this issue \citep{hernandez2017survey}, we include a review of these works in Appendix \ref{ap:related_work} for completeness. However, we note that the main purpose of this paper is to scale MARL up to systems with many agents. Hence, we are not concerned here with  comparing our method with those that exclusively target non-stationarity, especially given that, for scalability reasons, these cannot be applied to the high-dimensional problems we consider here.

\section{Preliminaries}\label{sec:background}


The type of problems we describe in the introduction can be formulated as factored partially observable stochastic games \citep{Hansen04AAAI}, which are defined as follows.
\begin{definition}[fPOSG]
A factored partially observable stochastic game (fPOSG) is a tuple $\langle N, S, A, T, \{R_i\}, \Omega, \{O_i\}\rangle$ where $N = \{1, ..., n\}$ is the set of $n$ agents, $S$ is the set of $j$ state variables $S = \{S^1, ..., S^k\}$, such that every state $s^t \in \times_{j=1}^k S^j$ is a $k$-dimensional vector $s^t = \langle s^{1,t}, ..., s^{k,t} \rangle$, $A = \times_{i\in N} A_i$ is the set of joint actions $a^t = \langle a^t_1, ..., a^t_n \rangle$, with $A_i$ being the set of actions for agent $i$, $T$ is the transition function, with $T(s^{t+1}|s^t,a^t)$, $R_i(s^t,a_i^t)$ is the immediate reward for agent $i$, $\Omega = \times_{i \in N} \Omega_i$ is the set of joint observations $o^t_i =\langle o^t_1, ..., o^t_n \rangle$, with $\Omega_i$ being the set of observations for agent $i$, and $O_i$ is the observation function for agent $i$, $O_i(o_i^t|s^t)$.
\end{definition}


Solving the fPOSG implies finding the policy $\pi_i$ for each agent $i$ that maximizes the expected return $G^t$; as defined in \citet{SuttonBarto98}. However, agents receive only partial observations $o_i$ of the true state $s$, which are not necessarily Markovian. Therefore, optimal policies are in general history-dependent in a POSG. Hence, we define the agents' policies $\pi_i(a_i^t|h_i^t)$ as mappings from action-observation histories (AOH), $h_i^t = \langle o_i^1, a_i^1, ..., a_{i}^{t-1}, o_i^t\rangle$, to probability distributions over actions, such that agent $i$'s optimal policy $\pi^*_i$ is the one that for every AOH $h_i^t$ selects the action with the highest $Q$-value $Q^{\pi_i^*}(h_i^t, a_i^t) =  \EX \left[G^t \mid h_i^t, a_i^t, \pi_i^* \right]$.

Given the structural assumptions we made in the introduction about the agents' only being able to observe and be rewarded for what occurs in their local neighborhood, we can narrow down the problem formulation and work with a specific class of fPOSGs called local-form fPOSGs \citep{oliehoek2021sufficient}, which better encompass the problems we consider here.
\begin{definition}[Local-form fPOSG]
\label{def:local-FPOSG}
A Local-form fPOSG is a fPOSG where $O_i$ and $R_i$ depend only on a subset of $m$ state variables $X_i = \{X_i^1, ..., X_i^m\} \subseteq S$, with $m \leq k$ (number of state variables), and agent $i$'s local states $x_i \in \times_{j=1}^m X_i^j$ being vectors $x^t_i = \langle x_i^{1,t}, ..., x_i^{k,t}\rangle$, such that $O_i(o_i^t|s^t) = \dot{O}_i(o_i^t|x_i^t)$ and $R_i(s^t, a_i^t) = \dot{R}_i(x_i^t, a_i^t)$, where $\dot{O}_i$ and $\dot{R}_i$ are the local observation and reward functions for agent $i$.
\end{definition}

\subsection{Influence-Based Abstraction}

We now describe the IBA framework \citep{oliehoek2021sufficient} which intends to simplify the local-form fPOSG formulation by exploiting its structural properties. The framework assumes that there is a single agent $i$ learning at a time while all other agents' policies $\pi_{-i}$ are fixed. Hence, from the perspective of agent $i$, the problem reduces to a POMDP \citep{Kaelbling96JAIR} where states are pairs $\langle s^t, h_{-i}^t \rangle$, with $h_{-i}^t$ being the AOHs of all agents but agent $i$ \citep{Nair03IJCAI}.
Finding a locally optimal solution to the Local-form fPOSG can be approached by ‘alternating maximization’ or ‘coordinate ascent’. That is, sequentially iterating over all agents,  possibly multiple times, and solving their respective POMDPs \citep{Nair03IJCAI, Oliehoek16Book}.

Looking at the definition of Local-form fPOSG, one can argue that, when solving for agent $i$, sampling actions from the policies of all the other agents $\pi_{-i}(a_{-i}^t| h_{-i}^t)$ and simulating the transitions $T(s^{t+1}|s^t, a^t)$ of the full set of state variables
is unnecessary, and while doing so is possible in small problems, it might become computationally intractable in large domains with many agents. Instead, we can define a new transition function $\bar{T}_i$ that models only agent $i$'s local state variables $x_i$, $\bar{T}_i(x_i^{t+1}|x_i^t, a_i^t)$.
The problem is that $x_i^{t+1}$ may still depend on the other agents' actions $a_{-i}^t$ and the  non-local state variables $S \setminus X_i$, which means that $\bar{T}_i$ is not well defined.
Fortunately, in many problems, only a fraction of the non-local state variables will \emph{directly influence} agent $i$'s local region.

The diagram on the left of Figure \ref{fig:pomdp+ips} is a Dynamic Bayesian Network (DBN) \citep{pearl88, boutilier1999decision} describing a particular instance of the transition dynamics for a generic agent $i$ in a local-form fPOSG. Agent $i$'s local region, 
corresponds to the variables that lie within the red box, $x_i \in X_i = \{ X_i^1,X_i^2 \}$. The diagram also shows the non-local variables, known as influence sources $u_i \in U_i \subseteq S \setminus X_i$, that influence the local region directly. The three dots on the top indicate that there can be, potentially many, other non-local variables in $S$ affecting the local variables $X_i$. These are denoted by $y_i \in Y_i \subseteq S \setminus X_i \cup U_i$. The diagram also shows that agent $j$ can affect agent $i$'s local region through its actions $a_j \in A_{j}$. However, both $Y_i$ and $A_{j}$ can only influence $X_i$ via $U_i$. Hence, given $u_i^t$, $x_i^{t+1}$ is conditionally independent of $y_i^t$ and $a_j^t$, therefore $P(x_i^{t+1}|x_i^t, u_i^t, y_i^t, a_j^t) = P(x_i^{t+1}|x_i^t, u_i^t)$. 

The above implies that by inferring the value of the influence source $u_i^t$, we can monitor the influence of the other agents and the non-local state variables and thus compute the local state transitions. This can be done by keeping track of the action-local-state history (ALSH) $l_i^t = \langle x_i^1, a_i^1 ..., a_i^{t-1}, x_i^t\rangle$.

\begin{definition}[IALM] An influence-augmented local Model (IALM) for agent $i$ is a tuple $\langle X_i, U_i, A_i, \dot{T}_i, \dot{R}_i, \Omega_i, \dot{O}_i, I_i\rangle$, with local states $x_i \in \times_{j \in |X_i|} X_i^j$, influence sources $u_i \in \times_{j \in |U_i|} U_i^j$, local transition function $\dot{T}_i(x_i^{t+1}|x_i^t, u_i^t, a_i^t)$, local observation function $\dot{O}_i(o_i^{t+1}|x_i^{t+1})$, local reward function $\dot{R}_i(x_i^t, a_i^t)$,  and influence distribution $I_i(u_i^t|l_i^t)$.
\label{def:IALM}
\end{definition}
Using the IALM we can compute agent $i$'s local transitions as
\begin{equation}
\small
\begin{split}
    P(x_i^{t+1}| l_i^t, a_i^t) = 
    \sum_{u^t}  \dot{T}_i(x_i^{t+1}|x_i^t, u_i^t, a_i^t)I_i(u_i^t|l_i^t).
    \label{eq:IALM}
\end{split}
\end{equation}
Note that, as opposed to the Local-form fPOSG, the transition function $\dot{T}_i$ in the IALM is defined purely in terms of the local state variables and the influence sources. Moreover, since $u_t$ \emph{d-separates} \citep{Bishop06book} $x_t$ from $y_t$, we only need to maintain a belief over $u_i^t$, $I_i(u_i^t|l_i^t)$, rather than over the full set of of state variables $s^t$ and other agents' histories $h_{-i}^t$, $P(s^t, h_{-i}^t| l_i^t)$. All in all, this translates into a much more compact, yet \emph{exact} representation of the problem \citep{oliehoek2021sufficient}, which should be computationally much lighter than the original Local-form fPOSG.
\begin{figure*}[t]
\vspace{-10pt}
\centering
\includegraphics[width=0.9\textwidth]{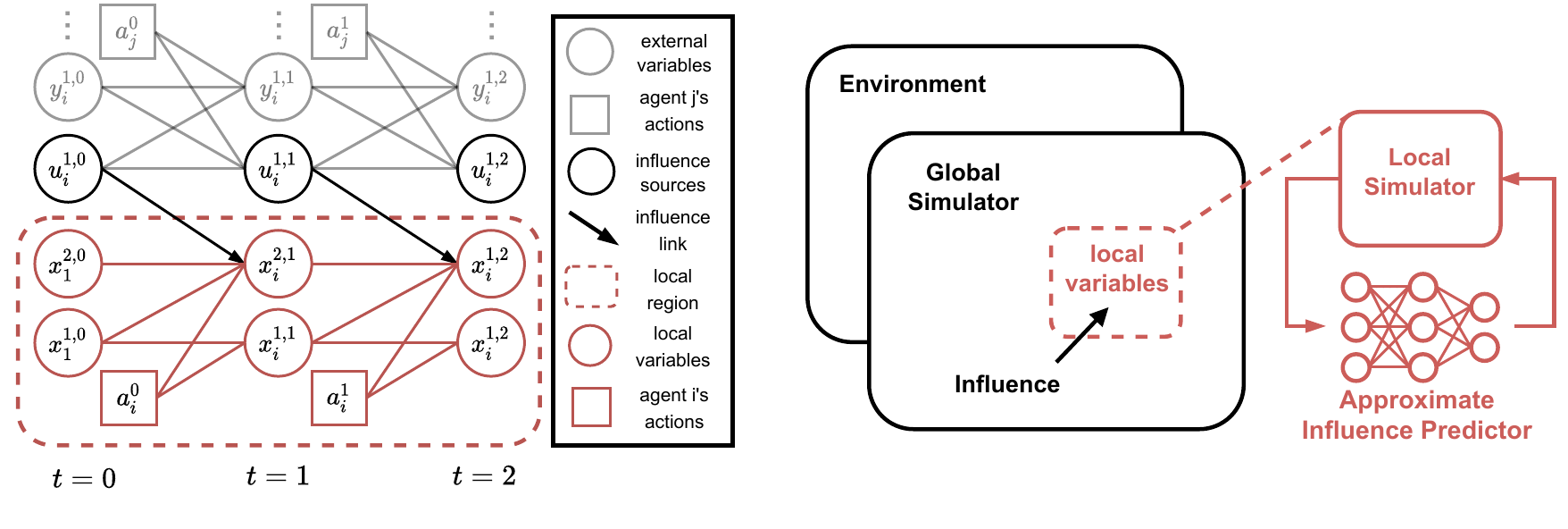}
\vspace{-10pt}
\caption{\textbf{Left:} A Dynamic Bayesian Network showing agent $i$'s transition dynamics in a local-form fPOSG prototype. \textbf{Right:} A conceptual diagram of the IALS.}
\label{fig:pomdp+ips}
\vspace{-10pt}
\end{figure*}
\subsection{Influence-Augmented Local Simulators}\label{sec:IALS}
Here we briefly describe how the IALM formulation can be used in practice to build IALSs \citep{suau2022influence}, which consist of a \emph{local simulator} and an \emph{approximate influence predictor}.


\paragraph{Local simulator (LS):}
The LS is an abstracted version of the environment that only models a small portion of it. As opposed to a global simulator (GS), which should closely reproduce the dynamics of every state variable, the LS focuses on characterizing the transitions of those variables $X_i$ that agent $i$ directly interacts with, $\dot{T}_i(x_i^{t+1}|x_i^t, u_i^t, a_i^t)$. 

\paragraph{Approximate influence predictor (AIP): } The AIP monitors the interactions between agent $i$'s local region $X_i$, the external variables $Y_i$, and the other agents' actions $A_{-i}$, by estimating $I_i(u_i^t|l_i^t)$. 
Since, due to combinatorial explosion, computing the exact probability $I_i(u_i^t|l_i^t)$ is generally intractable \citep{oliehoek2021sufficient}, a neural network is used instead to approximate the influence distribution. Thus, we write $\hat{I}_{\theta_i}$ to denote agent $i$'s AIP,
where $\theta_i$ are the network parameters.
The AIP
$\hat{I}_{\theta_i}$ is trained on a dataset $D_i$ of $N$ samples of the form $(l^t_i, u^t_i)$ collected from the GS.
Since the role of the AIP is to estimate the conditional probability of the influence sources $u_i^t$ given the past ALSH,
we can formulate the task as a classification problem and optimize the network using the expected cross-entropy loss 
\citep{Bishop06book}.


\section{Distributed Influence-Augmented Local Simulators}\label{sec:DIALS}
As mentioned in the previous section, local-form fPOSGs are solved iteratively in the IBA framework. This means that only a single agent can update its policy at a time. Here, we relax this assumption and discuss the advantages and disadvantages of simultaneous learning. Proofs for all the theoretical results in this section can be found in Appendix \ref{ap:proofs}.

\begin{wrapfigure}{r}{0.5\textwidth}
\vspace{-40pt}
  \begin{center}
    \includegraphics[width=0.5\textwidth]{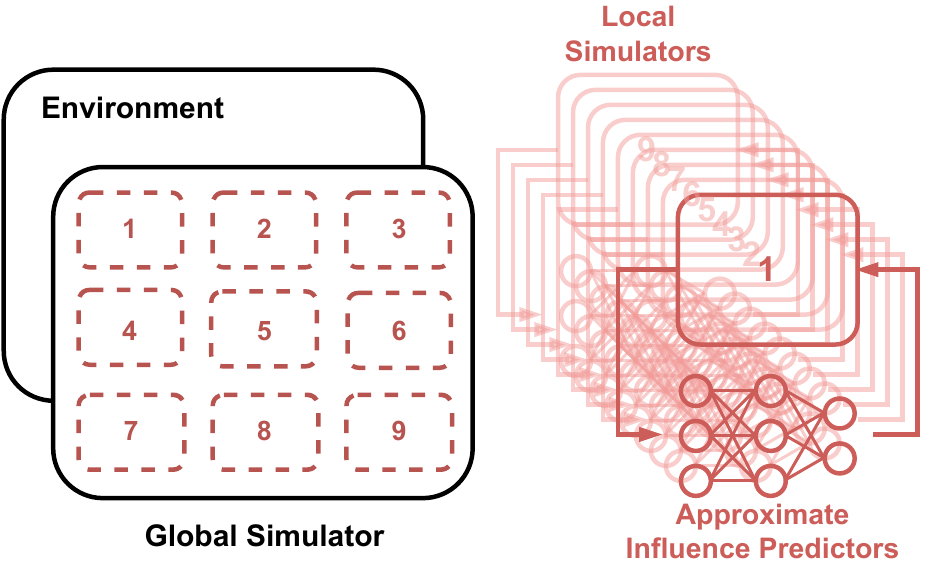}
  \end{center}
  \vspace{-5pt}
  \caption{A conceptual diagram of the DIALS}
  \vspace{-10pt}
\end{wrapfigure}
\subsection{Enabling Parallelization}\label{sec:parallelization}
The main reason to disallow simultaneous learning is that changes in the other agents' policies can affect agent $i$'s influence distribution $I_i(u_i^t|l_i^t)$, which may become non-stationary. This renders previously computed influences useless because they no longer capture the true response of the global system. 

This restriction, however, prevents IBA from unlocking its full potential. The fact that each agent's IALS is independent of the others means that the computations can be distributed among different processes that can run in parallel. Hence, putting aside the non-stationarity issue, and assuming no overhead costs in spawning an increasing number of processes, the total runtime of the method would stay constant if the dimensionality of the global system grew, either because the number of non-local variables or the number of agents increased. This is in contrast to having agents learn simultaneously in the same GS, in which case larger environments imply longer runtimes. Moreover, since each IALS simulates only a portion of the environment the total amount of memory space needed would be split among the different processors. Hence, we could run the simulation on multiple machines with small memory rather than one big machine with very large memory.


In principle, one could prevent the AIPs from becoming stale by simply updating all $\{\hat{I}_{\theta_i}(u_i^t|l_i^t)\}_{i\in N}$ every time any of the other agents changes its policy. 
However, this creates a difficult moving target problem and makes the whole method very inefficient since, especially in deep RL, policies are updated very frequently.
Fortunately, as we argue in the following, in many cases, paying the extra cost of retraining the AIPs is neither necessary nor desirable.

\subsubsection{Multiple joint policies may induce the same influence distribution}
In the following, we show that multiple joint policies may often map onto the same influence distribution $I_i(u_i^t|l_i^t) \in \Psi_i$ for agent $i \in N$.
\begin{restatable}{lemma}{onetoone}
Let $\Pi = \times_{i \in N} \Pi_i$ be the product space of joint policies with $\Pi_i$ being the set of policies for agent $i$. Moreover, let $\Psi = \times_{i\in N} \Psi_i$ be the product space of joint influences, with $\Psi_i$ being the set of influence distributions for agent $i$. Every joint policy $\pi \in  \Pi$ induces exactly one influence distribution $I_i(u_i^t|l_i^t) \in \Psi_i$ for every agent $i \in N$.
\label{prop:onetoone}
\end{restatable}

\begin{restatable}{proposition}{manytoone}
The space of joint policies $\Pi = \times_{i \in N} \Pi_i$  is necessarily greater than or equal to the space of joint influences $\Psi = \times_{i \in N} \Psi_i$,  $|\Pi| \geq |\Psi|$. Moreover, there exist local-form fPOSGs for which the inequality is strict.
\end{restatable}

The advantages of this result were shown empirically by \cite{Witwicki10ICAPS}, who demonstrated that planning times can be reduced by searching the space of joint influences rather than the space of joint policies, which is often much larger. In fact, in the extreme case of local transition independence \citep{Becker03AAMAS},\footnote{As opposed to IBA, \citet{Becker03AAMAS} assume agents are tied by a shared global reward.} we have that for all joint policies $\pi$ there is a single $\{I_i\}_{i \in N}$
\begin{restatable}{corollary}{corollary1}
Let agent $i$'s influence sources $u_i^t$ be independent of the other agents' actions $a_{-i}$. Then, for any joint policy $\pi \in \Pi$, there is a unique influence distribution $I^*_i \in \Psi_i$ for every agent $i \in N$ and $|\Pi| \gg |\Psi| = 1$.
\label{col:col1}
\end{restatable}
 The result above implies that, in this particular case, we would only need to train the AIPs once at the beginning. Although we do not expect the situation in Corollary \ref{col:col1} to be the norm, we do believe that in many scenarios, such as in the two environments we explore here,  we would not need to retrain the AIPs very often because similar joint policies will influence the local regions in very similar, if not in the same, ways. Furthermore, the next result shows that even when this is not the case, an outdated $I_i$ computed from an old joint policy might still produce the same optimal policy for agent $i$.

\subsubsection{Multiple influence distributions may induce the same optimal policy}

We use the simulation lemma \citep{kearns2002near} to prove that if two influence distributions are similar enough they will induce the same optimal policy.
\begin{restatable}{lemma}{propositiontwo}
Let $M_i^1$ and $M_i^2$ be two IALMS differing only on their influence distributions $I_i^1(u_i^t|l_i^t)$ and $I_i^2(u_i^t|l_i^t)$. Let $Q^{\pi_i}_{M_i^1}$ and $Q^{\pi_i}_{M_i^2}$ be the value functions induced by $M_i^1$ and $M_i^2$ for the same $\pi_i$. If $I^1_i$ and $I^2_i$ satisfy
\begin{equation}
\small
    \sum_{l_i^t,u_i^t} P(l_i^t|h_i^t)  \left|I^1_i(u_i^t| l_i^t) - I^2_i(u_i^t| l_i^t) \right| \leq \xi 
\text{, then }
    \left|Q^{\pi_i}_{M_i^1}(h_i^{t}, a_i^t) - Q^{\pi_i}_{M_i^2}(h_i^{t}, a_i^t)\right| \leq \bar{R}\frac{(H-t)(H-t+1)}{2}\xi \qquad 
\end{equation}
for all $\pi_i$, $h_i^t$, and $a_i^t$, where $H$ is the horizon and $\bar{R} = || R||_\infty$ 
\label{lemma:two}
\end{restatable}
Intuitively, Lemma \ref{lemma:two} shows that the difference in value between $M_i^1$ and $M_i^2$ is upper-bounded by the maximum difference between $I_i^1$ and $I_i^2$ times a constant.
Actually, if the \emph{action-gap} \citep{farahmand2011action} (i.e. value difference between the best and the second best action) in one of the IALMs is larger than twice the difference between the $Q_{M_1}^{\pi_i}$ and $Q_{M_2}^{\pi_i}$ the IALMs share the same optimal policy.
\begin{restatable}{theorem}{theoremtwo}
Let $M_i^1$ and $M_i^2$ be two IALMS differing only on their influence distributions $I_i^1(u_i^t|l_i^t)$ and $I_i^2(u_i^t|l_i^t)$. $M_i^1$ and $M_i^2$ induce the same optimal policy $\pi^*$ if, for some $\Delta$, 
\begin{equation}
\small
    Q^{\pi_i^*}_{M_i^1}(h_i^t, \bar{a}_i^t) - Q^{\pi_i^*}_{M_i^1}(h_i^{t}, {\hat{a}_i^t}) > 2\Delta \quad \forall h_i^t, \hat{a}_i^t \neq \bar{a}_i^t
    \text{ with }
  \left|Q^{\pi_i}_{M_i^1}(h_i^t, a_i^t) - Q^{\pi_i}_{M_i^2}(h_i^{t}, a_i^t)\right| \leq \Delta  \quad \forall h_i^t, a_i^t, \pi_i,
\end{equation}
where
$\bar{a}_i^t = \argmax_{a_i^t} Q_{M_i^1}^{\pi_i^*}(h_i^t, a_i^t)$
\label{thm:two}
\end{restatable}
Combining Lemma \ref{lemma:two} and Theorem \ref{thm:two}, we see that, because the difference in value between $M_i^1$ an $M_i^2$ depends on $\xi$ (Lemma \ref{lemma:two}),
the closer the distributions $I_i^1$ and $I_i^2$ are, the more likely it is that  $M_i^1$ and $M_i^2$ share the same optimal policy. Note that we have no control over the action gap as it is domain-dependent. In some domains, the gap might be large and we can be more relaxed about not retraining the AIPs. In some others, the gap might be small and we may need to retrain the AIPs more frequently.

\subsection{Algorithm}\label{sec:algo}
After the analysis above, we are now ready to present our method, which we call Distributed Influence-Augmented Local Simulators (DIALS). Algorithm \ref{alg:DIALS} describes how we can train multi-agent systems with DIALS. As mentioned earlier, the key advantage of using DIALS is that simulations can be distributed among different processes, and thus training can be fully parallelized. This enables MARL to scale to very large systems with many learning agents. Moreover, following from our theoretical results, AIP training, which can also be done in parallel, is performed only every certain number of timesteps. The hyperparameter $F$ in Algorithm \ref{alg:DIALS} controls the AIPs' training frequency. The effect of $F$ on the learning performance is empirically investigated in Section \ref{sec:experiments}.
\algblock{ParFor}{EndParFor}
\algnewcommand\algorithmicparfor{\textbf{in parallel, for}}
\algnewcommand\algorithmicpardo{\textbf{do}}
\algnewcommand\algorithmicendparfor{\textbf{end\ for}}
\algrenewtext{ParFor}[1]{\algorithmicparfor\ #1\ \algorithmicpardo}
\algrenewtext{EndParFor}{\algorithmicendparfor}
\vspace{-5pt}
\begin{algorithm}
\caption{MARL with DIALS}
\begin{algorithmic}[1]
\small
  \State Initialize policies $\{\pi_i\}_{i\in N}$ and AIPs $\{\hat{I}_{\theta_i}\}_{i \in N}$
  \Repeat
  \State Collect datasets $\{D_i\}_{i \in N}$ from GS
  \Comment{See Algorithm \ref{alg:collect} in Appendix \ref{ap:algorithms}}
  \ParFor{$i \in N$} 
  \State  Train AIP $\hat{I}_{\theta_i}$ on dataset $D_i$
  \Comment{See Section \ref{sec:IALS}}
  \EndParFor
  \ParFor{$i \in N$} 
  \For{$F$ steps}
  \Comment{$F$ is the AIPs' training frequency}
  \State  Simulate trajectories with IALS $\langle \dot{T}_i, \dot{R}_i, \dot{O}_i, \hat{I}_{\theta_i}\rangle$
  \Comment{See Algorithm \ref{alg:sample} in Appendix \ref{ap:algorithms}}
  \State Train policy $\pi_i$
  \Comment{Using any standard RL method}
  \EndFor
  \EndParFor
  \Until{end of training}
\end{algorithmic}
\label{alg:DIALS}
\end{algorithm}
\vspace{-5pt}

\subsection{Mitigating the negative effects of simultaneous learning}\label{sec:mitigating}

The results in the previous section showed that the AIPs may not need to be retrained every single time the policies are updated, as the influence distributions may often stay the same or vary only a little. Yet, we now argue that, even when changes in the joint policy do affect the influence distributions $I_i(u_i^t | l_i^t)$ significantly, it may be advantageous not to retrain the AIPs. 

First, we have already mentioned that when all agents learn simultaneously in the same simulator the transition dynamics often look non-stationary from the perspective of each individual agent. This may result in sudden performance drops caused by oscillations in the value targets \citep{Claus98AAAI}. In contrast, when using independent IALS to train our agents, the transition dynamics remain stationary unless we update the AIPs. Hence, by not updating the AIPs too frequently, we get a biased but otherwise more consistent learning signal that the agents can rely on to improve their policies.

Second, we posit that the poor empirical convergence of many off-the-shelf Deep RL methods \citep{hernandez2017survey, yu2021surprising} is also because stochastic gradient descent updates often result in policies that perform worse than the previous ones. Thus, when learning together, agents may try to adapt to other agents' poor performing policies. These policies, however, are likely to be temporary as they are just a result of the inherent stochasticity of the learning process. Similarly, in many environments, agents shall take exploratory actions before they can improve their policies, which may also negatively impact cooperation if they learn simultaneously \citep{zhang2009integrating, zhang2010self}. In our case, we can again benefit from the fact that the AIPs need to be purposely retrained,
and do so only when the policies of the other agents have improved sufficiently. 

Even though further theoretical analysis would be needed to be more conclusive about the benefits of using independent simulators, the observations above give reasons to believe that what we initially described as a problem may in fact be an advantage. This view is also supported by our experiments.





\section{Experiments}\label{sec:experiments}

The goal of the experiments is to: (1) test whether we can reduce training times by replacing GS with DIALS, (2) investigate how the method scales to large environments with many learning agents, (3) evaluate the convergence benefits of using separate simulators to train agents rather than a single GS, and (4) study the effect of the AIPs' training frequency $F$ on the agents' learning performance.
\subsection{Experimental setup}\label{sec:setup}
Agents are trained independently with PPO \citep{schulman2017proximal}\footnote{The vanilla PPO algorithm with decentralized value functions (independent PPO; IPPO) has been shown to perform exceptionally well on several multi-agent environments \citep{de2020independent, yu2021surprising}.} on (1) the global simulator (GS), (2) distributed influence-augmented local simulators (DIALS) with AIPs trained periodically on datasets collected from the GS using the most recent joint policy, (3) DIALS with untrained AIPs (untrained-DIALS).

To measure the agent's performance, training is interleaved with periodic evaluations on the GS. The results are averaged over 10 random seeds on all except on the largest scenarios ($10 \times 10$) for which, due to computational limitations we could only run 5 seeds. We report the mean return of all learning agents. We also compare the simulators in terms of total runtime. For DIALS this includes the agents' training time, the AIPs' training time, and the time for data collection.
\subsection{Environments}
\paragraph{Traffic control}\label{sec:traffic}
The first domain we consider is a multi-agent variant of the traffic control benchmark proposed by \cite{vinitsky2018benchmarks}. In this scenario, agents are requested to manage the lights of a big traffic network. Each agent controls a single traffic light and can only observe cars when they are inside the intersection's local neighborhood. Their goal is to maximize the average speed of cars within their respective intersections. To demonstrate the scalability of the method we evaluate DIALS on four different variants of the traffic network with $4$, $25$, $49$, and even $100$ intersections (agents). A screenshot of the traffic network with $25$ intersections is shown in Appendix \ref{ap:screenshots}. The GS  and LS are built using Flow (MIT License) \citep{wu2017flow} and SUMO (Eclipse Public License Version 2) \citep{SUMO2018}. The GS simulates the entire traffic network while each LS models only the local neighborhood of each intersection $i \in N$ (Figure \ref{fig:local_sim} in Appendix \ref{ap:screenshots}). We use the same LS for every agent-intersection but since, depending on where they are located, they are influenced differently by the rest of the traffic network, we have separate AIPs, $\{I_{\theta_i}\}_{i \in N}$, for each them. These are feedforward neural networks with the same architecture but different weights $\theta_i$, trained periodically with frequency $F$ on datasets $\{D_i\}_{i \in N}$  collected from the GS. The influence sources $u^i_t$ are binary variables indicating whether or not a car will be entering from each of the four incoming lanes. 


\paragraph{Warehouse Commissioning}
The second domain we consider is a warehouse commissioning task \citep{suau2022influence}. A team of robots (blue) needs to fetch the items (yellow) that appear with probability $0.02$ on the shelves (dashed black lines) of the warehouse (see Figure \ref{fig:global_sim} in Appendix \ref{ap:screenshots}). Each robot has been designated a $5 \times 5$ square region and can only collect the items that appear on the shelves at the edges. The regions overlap so that each of the $4$ item shelves in a robot's region is shared with one of its $4$ neighbors. The robots receive a reward between $[0, 1]$ when collecting an item. The exact value depends on how old the item is compared to the other items in their region. This is to encourage the robots to collect the oldest items first. The robots receive as observations a bitmap encoding their own location and a set of $12$ binary variables that indicate whether or not a given item needs to be collected. The robots, however, cannot see the location of the other robots even though all of them are directly or indirectly influencing each other through their actions. We built four variants of the warehouse with $4$, $25$, $49$, and $100$ robots (agents). A screenshot of the warehouse with $25$ robots is shown in Appendix \ref{ap:screenshots}. The GS simulates the entire warehouse while the LS models only a $5 \times 5$ square region (Figure \ref{fig:local_sim} in Appendix \ref{ap:screenshots}). We use the same LS for every robot (agent) but since depending on where they are located they are influenced differently by the rest of the robots, we have separate AIPs, $\{\hat{I}_{\theta_i}\}_{i \in N}$, for each of them. These are GRUs \citep{Cho2014Learning} with the same architecture but different weights $\theta_i$, which we train periodically with frequency $F$ on datasets $\{D_i\}_{i \in N}$ collected from the GS. Robot $i$'s influence sources $u_i^t$ encode the location of the four neighbor robots. If its AIP $\hat{I}_{\theta_i}$ predicts that any of the neighbor robots is at one of the $12$ cells within its region, and there is an active item on that cell, that item is removed and robot $i$ can no longer collect it.

\subsection{Results}\label{sec:results}
\paragraph{GS vs. DIALS}
The two plots on the left of Figures \ref{fig:traffic_plot} and \ref{fig:warehouse_plot} show the average return as a function of the number of timesteps obtained with GS, DIALS, and untrained-DIALS on the $4$-agent traffic and warehouse environments. Shaded areas indicate the standard error of the mean. Agents are trained for 4M timesteps on all three simulators.
The results reveal that, while agents trained on DIALS seem to converge steadily towards similar high-performing policies, agents trained with the GS often get stuck in local minima, hence the poor mean episodic reward and large standard error obtained with the GS relative to that of the DIALS. In contrast, the low performance of agents trained with the untrained-DIALS indicates that estimating the influences correctly is important for learning good policies. It is worth noting that the gap between the GS and the DIALS is larger in Figure \ref{fig:warehouse_plot} than in Figure \ref{fig:traffic_plot}. We posit that this is because, in the warehouse domain, agents are more strongly coupled.
For comparison, the dashed-black lines in the plots on the left of Figures \ref{fig:traffic_plot} and \ref{fig:warehouse_plot} show the performance of hand-coded policies. For the traffic domain, we used fixed traffic light controllers that were extensively optimized by \citet{wu2017flow}. For the warehouse domain, we hand-coded policies that follow the shortest path toward the oldest item in the agent's region. 
The learning curves for the other scenarios are provided in Appendix \ref{ap:results} together with further discussion on these results.\footnote{A video showing the GS of the traffic network and one of the IALS is provided at \url{https://youtu.be/DgVE6OIQQz8}.}\footnote{A video showing how agents trained with DIALS perform on the 100-agent variant of the traffic scenario is provided at \url{https://youtu.be/G9EthZ-G3vo}.}
\begin{figure}
\vspace{-10pt}
     \centering
     \begin{subfigure}[b]{0.49\textwidth}
         \centering
         \includegraphics[width=\textwidth]{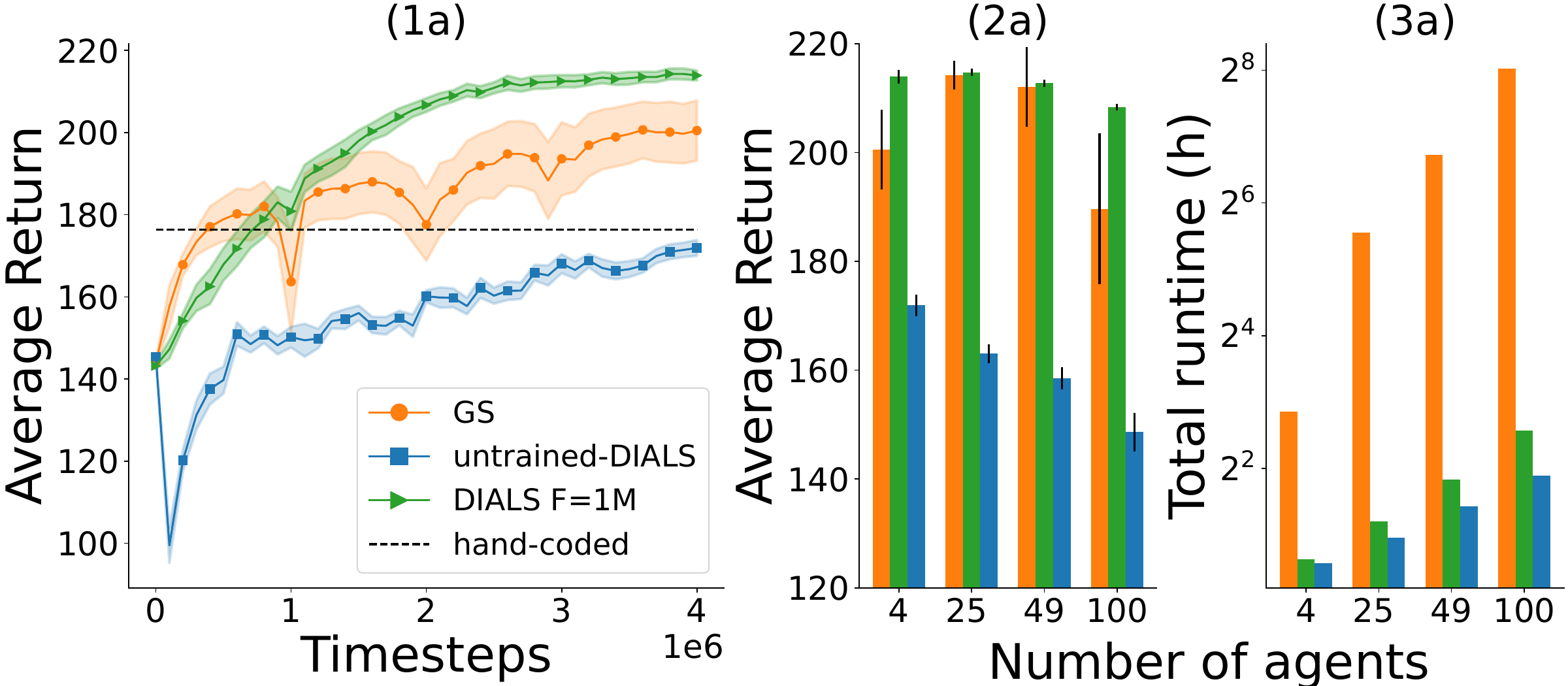}
         \caption{Traffic}
         \label{fig:traffic_plot}
     \end{subfigure}
     \hfill
     \begin{subfigure}[b]{0.49\textwidth}
         \centering
         \includegraphics[width=\textwidth]{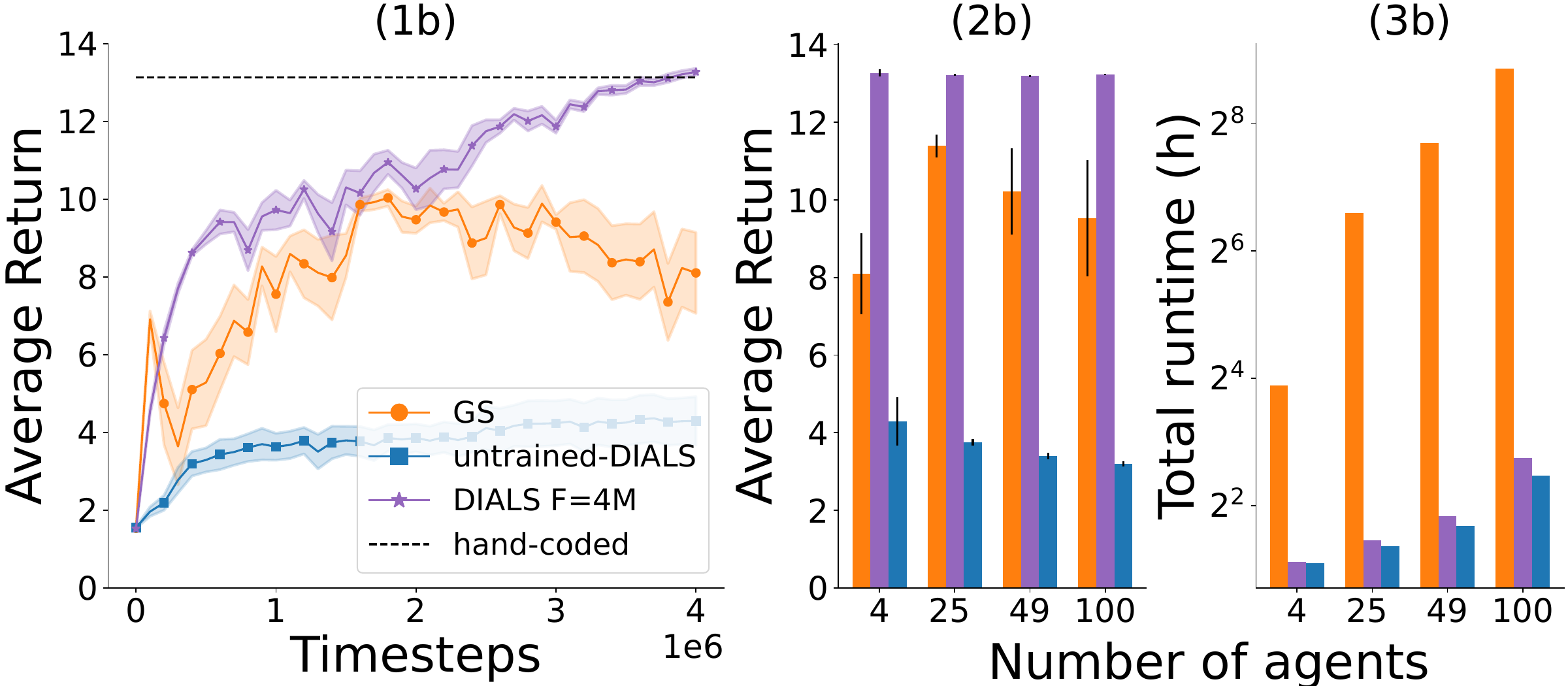}
         \caption{Warehouse}
         \label{fig:warehouse_plot}
     \end{subfigure}
      \caption{\textbf{(1a) and (1b)} Learning curves with the three simulators on the $4$-intersection traffic and 4-robot warehouse environments. \textbf{(2a) and (2b):} Final average return of agents trained with the three simulators for 4M timesteps. \textbf{(3a) and (3b):} Total runtime of training with the three simulators for 4M timesteps. The $y$-axis is in $\log_2$ scale.}
      \label{fig:plots}
\end{figure}
\paragraph{Scalability} The benefits of parallelization are more apparent when moving to larger environments. The two bar plots in Figures \ref{fig:traffic_plot} and \ref{fig:warehouse_plot} depict the final average return and the total run time of training $4$, $25$, $49$, and $100$ agents on the two tasks for $4$M timesteps. The plots show that DIALS scales far better than the GS to larger problem sizes, while also yielding better-performing policies. For example, training $\mathbf{100}$ \textbf{agents} on the traffic network takes less than $\mathbf{6}$ \textbf{hours} with the DIALS, whereas training them with the GS would take more than $\mathbf{10}$ \textbf{days}. This is a \textbf{speedup factor of }$\mathbf{40}$. In fact, since the maximum execution time allowed by our computer cluster is $1$ week, the results reported for the GS in the scenarios of size $10 \times 10$ do not correspond to 4M timesteps but the equivalent of $1$ week of training. We would also like to point out that, disregarding the overhead costs associated with multiprocessing, the DIALS runtime should remain constant independently of the problem size. However, 
 to update the AIPs, new samples are collected from the GS, which does increase the runtime. This explains the gap between DIALS and untrained-DIALS. That said, the number of samples needed to update the AIPs (80K for traffic and 10K for warehouse) is significantly lower than the samples needed to train the agents (4M), which is why the runtime difference between GS and DIALS is so large.
A table with a breakdown of the runtimes is given in Appendix \ref{ap:runtimes}.


\paragraph{AIPs' training frequency}
Our first results have already demonstrated that isolating the agents in separate simulators and not updating the AIPs too frequently can be beneficial for convergence. We now further investigate this phenomenon by evaluating the agents' learning performance for different values of the hyperparameter $F$. The two plots on the left of Figures \ref{fig:DIALS_comparison_traffic} and \ref{fig:DIALS_comparison_warehouse} show the learning curves for agents trained on DIALS where $F$ is set to $100$K, $500$K, $1$M, and $4$M timesteps. In the traffic domain, the gap between the green and the purple curve (Figure \ref{fig:DIALS_comparison_traffic}) suggests that it is important to retrain the AIPs at least every $1$M timesteps, such that agents become aware of changes in the other agents' policies. In contrast, in the warehouse domain (Figure \ref{fig:DIALS_comparison_warehouse}), we see that training the AIPs only once at the beginning (DIALS $F=4$M) seems sufficient. In fact, updating the AIPs too frequently (DIALS $F=100$K) is detrimental to the agents' performance. This is consistent with our hypothesis in Section \ref{sec:mitigating}. The plots on the right show the average cross-entropy (CE) loss of the AIPs  evaluated on trajectories sampled from the GS. As explained in Section \ref{sec:DIALS} since all agents learn simultaneously, the influence distributions $\{I(u_i^t |l_i^t)\}_{i \in N}$ are non-stationary. For this reason, we see that the CE loss changes as the policies of the other agents are updated. We can also see how the CE loss decreases when the AIPs are retrained, which happens more or less frequently depending on the hyperparameter $F$. Note that the CE not only measures the distance between the two probability distributions but also the absolute entropy. In the warehouse domain, the neighbor robots' locations become more predictable (lower entropy) as their policies improve. This explains why in the first plot from the right the CE loss decreases even though the AIPs are not updated. Also in the same plot, even though by the end of training DIALS $F=4$M is highly inaccurate, as evidenced by the gap between the purple and the other curves, it is still good enough to train policies that match the performance of those trained with DIALS $F=500$K and $F=1$M. This is in line with our results in Section \ref{sec:DIALS}. The same plots for the rest of the scenarios are provided in Appendix \ref{ap:results}.


\begin{figure}
     \centering
     \begin{subfigure}[b]{0.49\textwidth}
         \centering
         \includegraphics[width=\textwidth]{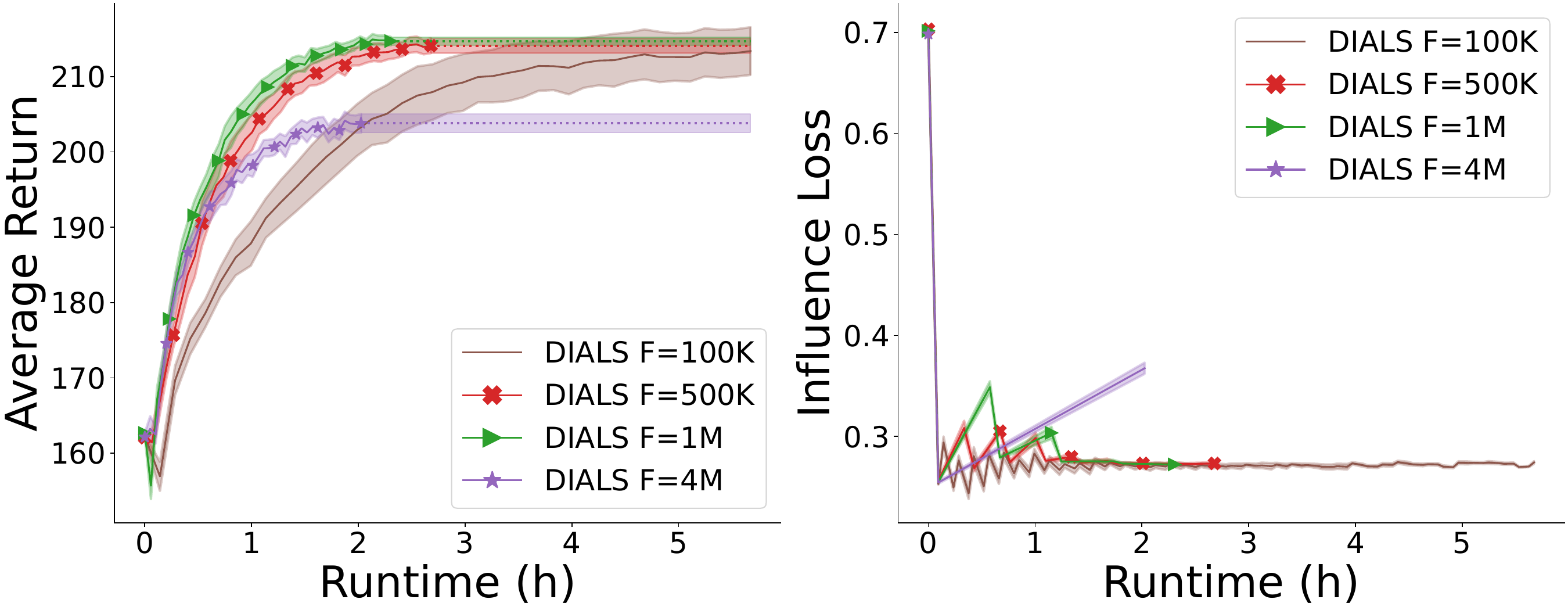}
         \caption{Traffic $25$ agents}
         \label{fig:DIALS_comparison_traffic}
     \end{subfigure}
     \hfill
     \begin{subfigure}[b]{0.49\textwidth}
         \centering
         \includegraphics[width=\textwidth]{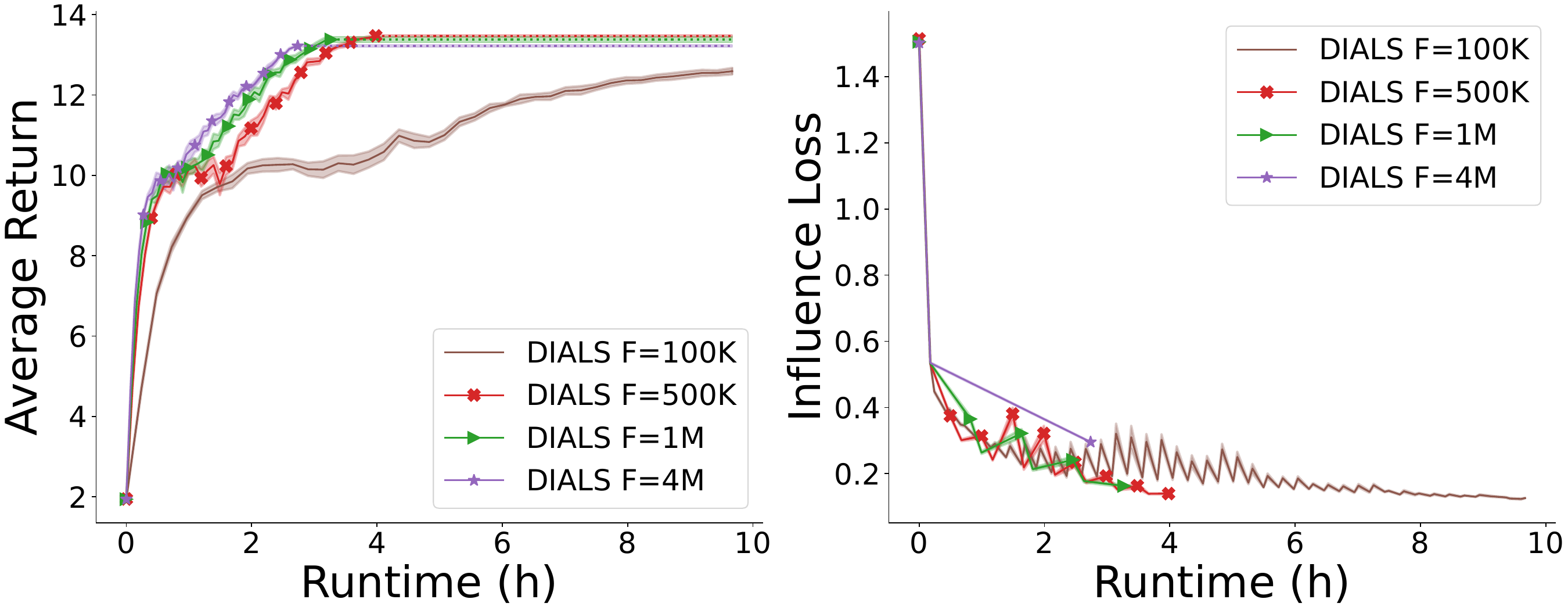}
         \caption{Warehouse $25$ agents}
         \label{fig:DIALS_comparison_warehouse}
     \end{subfigure}
     \caption{\textbf{Left (a) and (b):} Learning curves with DIALS for different values of $F$ on the $25$-agent versions of the two environments. \textbf{Right (a) and (b):} Influence CE loss as a function of runtime averaged over the $25$ AIPs.}
\end{figure}
\section{Scope and Limitations}

DIALS targets networked environments with well-defined local regions where the interactions between different regions occur through a limited number of variables. There is plenty of examples of domains that have this particular structure including traffic, heating and water systems, logistics, telecommunications, etc. Knowledge of the influence sources $U$ and how these affect the local regions is required for building the DIALS. In most cases, however, (as in the two environments we explored here) some domain knowledge suffices to be able to tell what the influence sources are. 

Moreover, having or being able to build high-fidelity local simulators of these local regions is also a requirement. Fortunately, there exist plenty of simulators of real systems that can readily be used such as, SUMO \citep{SUMO2018}, Robosuite \citep{zhu2020robosuite}, BRAX \citep{freeman2021brax}. There is also a lot of commercial software for building custom-made simulators such as Mujoco \citep{todorov2012mujoco}, or Unity \citep{juliani2018unity}. Also, note that most (if not all) of the work that has applied RL to real-world problems relies on simulation to train the policies offline \citep{bellemare2020autonomous, degrave2022magnetic}. Hence, we believe that DIALS can have a strong impact on many real-world applications.

Finally, although the experiments reveal that DIALS can considerably accelerate training times, it is also memory-demanding. As shown in Appendix \ref{ap:memory} (Table \ref{tab:memory}), the total memory usage with DIALS increases exponentially with the number of simulators/processes. There is thus a trade-off between fast computation and total memory needed. Note, however, that the memory is split among the different processes. Hence, rather than using a big machine with large memory, DIALS can run on several smaller ones with less memory.

\section{Conclusion}
This paper has offered a practical solution that allows training large networked systems with many agents in just a few hours. We showed how to factorize these systems into multiple sub-regions such that we could build distributed influence-augmented local simulators (DIALS). 

The key advantage of DIALS is that simulations can be distributed among different processes, and thus training can be fully parallelized. To account for the interactions between the different sub-regions, the simulators are equipped with approximate influence predictors (AIPs), which are trained periodically on real trajectories sampled from a global simulator (GS). We demonstrated that, although using DIALS agents learn simultaneously, training the AIPs very frequently is neither necessary nor desirable. Our results reveal that DIALS not only enables MARL to scale up but also mitigates the non-stationarity issues of simultaneous learning. 

Future work could analyze this phenomenon from a theoretical perspective, study how to adapt DIALS to more strongly coupled domains where frequent training of the AIPs is important, or design a method to directly estimate how the changes in the agents’ policies affect the influence distributions. This is so that the AIPs can readily be updated without having to run the GS to generate new samples.


\section*{Acknowledgements}
\begin{wrapfigure}{r}{0.28\linewidth}
    \vspace{-12pt}
    \hspace{5pt}
    \includegraphics[width=0.88\linewidth]{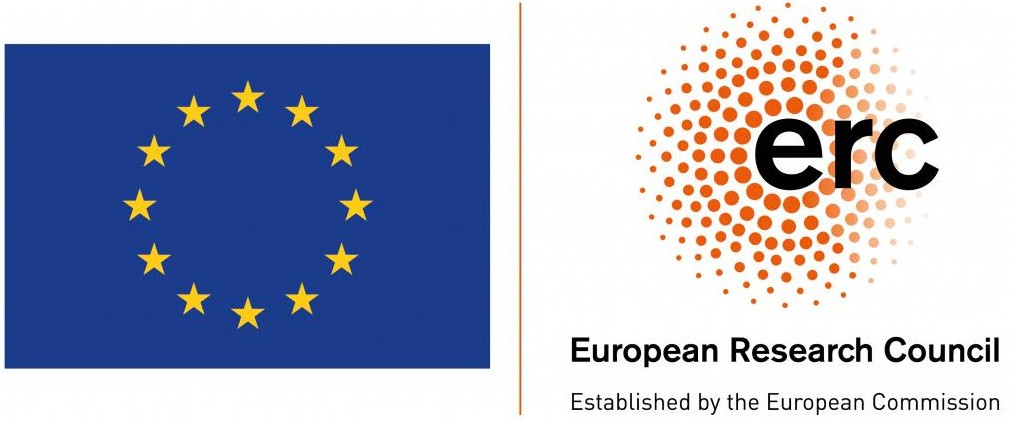}
\end{wrapfigure}
This project received funding from the European Research Council (ERC) 
under the European Union's Horizon 2020  research 
and innovation program (grant agreement No.~758824 \textemdash INFLUENCE). Mustafa Mert \c{C}elikok is partially funded by KAUTE Foundation - The Finnish Science Foundation for Economics and Technology.

\newpage
\bibliography{bibliography}

\providecommand{\noopsort}[1]{}\providecommand{\noopsort}[1]{}\providecommand{\noopsort}[1]{}\providecommand{\noopsort}[1]{}
\begin{thebibliography}{}

\bibitem[Becker et~al., 2003]{Becker03AAMAS}
Becker, R., Zilberstein, S., Lesser, V., and Goldman, C.~V. (2003).
\newblock Transition-independent decentralized {M}arkov decision processes.
\newblock In {\em Proc.\ of the International Conference on Autonomous Agents
  and Multiagent Systems}, pages 41--48.

\bibitem[Bellemare et~al., 2020]{bellemare2020autonomous}
Bellemare, M.~G., Candido, S., Castro, P.~S., Gong, J., Machado, M.~C., Moitra,
  S., Ponda, S.~S., and Wang, Z. (2020).
\newblock Autonomous navigation of stratospheric balloons using reinforcement
  learning.
\newblock {\em Nature}, 588(7836):77--82.

\bibitem[Bishop, 2006]{Bishop06book}
Bishop, C.~M. (2006).
\newblock {\em Pattern Recognition and Machine Learning}.
\newblock Springer.

\bibitem[Boutilier et~al., 1999]{boutilier1999decision}
Boutilier, C., Dean, T., and Hanks, S. (1999).
\newblock {Decision-theoretic planning: Structural assumptions and
  computational leverage}.
\newblock {\em Journal of Artificial Intelligence Research}, 11:1--94.

\bibitem[Castellini et~al., 2020]{castellini2020difference}
Castellini, J., Devlin, S., Oliehoek, F.~A., and Savani, R. (2020).
\newblock Difference rewards policy gradients.
\newblock {\em arXiv preprint arXiv:2012.11258}.

\bibitem[Cho et~al., 2014]{Cho2014Learning}
Cho, K., {van Merrienboer}, B., Gulcehre, C., Bougares, F., Schwenk, H., and
  Bengio, Y. (2014).
\newblock Learning phrase representations using rnn encoder-decoder for
  statistical machine translation.
\newblock In {\em Conference on Empirical Methods in Natural Language
  Processing (EMNLP 2014)}.

\bibitem[Claus and Boutilier, 1998]{Claus98AAAI}
Claus, C. and Boutilier, C. (1998).
\newblock The dynamics of reinforcement learning in cooperative multiagent
  systems.
\newblock In {\em Proc.\ of the National Conference on Artificial
  Intelligence}, pages 746--752.

\bibitem[Congeduti et~al., 2021]{Congeduti21AAMAS}
Congeduti, E., Mey, A., and Oliehoek, F.~A. (2021).
\newblock Loss bounds for approximate influence-based abstraction.
\newblock In {\em Proceedings of the Twentieth International Conference on
  Autonomous Agents and MultiAgent Systems}.

\bibitem[de~Witt et~al., 2020]{de2020independent}
de~Witt, C.~S., Gupta, T., Makoviichuk, D., Makoviychuk, V., Torr, P.~H., Sun,
  M., and Whiteson, S. (2020).
\newblock Is independent learning all you need in the starcraft multi-agent
  challenge?
\newblock {\em arXiv preprint arXiv:2011.09533}.

\bibitem[Dearden and Boutilier, 1997]{Dearden97AIJ}
Dearden, R. and Boutilier, C. (1997).
\newblock Abstraction and approximate decision-theoretic planning.
\newblock {\em Artificial Intelligence}, 89(1-2):219--283.

\bibitem[Degrave et~al., 2022]{degrave2022magnetic}
Degrave, J., Felici, F., Buchli, J., Neunert, M., Tracey, B., Carpanese, F.,
  Ewalds, T., Hafner, R., Abdolmaleki, A., de~Las~Casas, D., et~al. (2022).
\newblock Magnetic control of tokamak plasmas through deep reinforcement
  learning.
\newblock {\em Nature}, 602(7897):414--419.

\bibitem[Farahmand, 2011]{farahmand2011action}
Farahmand, A.-m. (2011).
\newblock Action-gap phenomenon in reinforcement learning.
\newblock {\em Advances in Neural Information Processing Systems}, 24.

\bibitem[Foerster et~al., 2018a]{Foerster18AAMAS_LOLA}
Foerster, J., Chen, R.~Y., Al-Shedivat, M., Whiteson, S., Abbeel, P., and
  Mordatch, I. (2018a).
\newblock Learning with opponent-learning awareness.
\newblock In {\em Proceedings of the Seventeenth International Conference on
  Autonomous Agents and MultiAgent Systems}.

\bibitem[Foerster et~al., 2018b]{Foerster18AAAI}
Foerster, J., Farquhar, G., Afouras, T., Nardelli, N., and Whiteson, S.
  (2018b).
\newblock Counterfactual multi-agent policy gradients.
\newblock In {\em Proceedings of the Thirty-Second AAAI Conference on
  Artificial Intelligence}.

\bibitem[Freeman et~al., 2021]{freeman2021brax}
Freeman, C.~D., Frey, E., Raichuk, A., Girgin, S., Mordatch, I., and Bachem, O.
  (2021).
\newblock Brax--a differentiable physics engine for large scale rigid body
  simulation.
\newblock {\em arXiv preprint arXiv:2106.13281}.

\bibitem[Hansen et~al., 2004]{Hansen04AAAI}
Hansen, E.~A., Bernstein, D.~S., and Zilberstein, S. (2004).
\newblock Dynamic programming for partially observable stochastic games.
\newblock In {\em Proc.\ of the National Conference on Artificial
  Intelligence}, pages 709--715.

\bibitem[He et~al., 2020]{he2020influence}
He, J., Suau, M., and Oliehoek, F. (2020).
\newblock Influence-augmented online planning for complex environments.
\newblock In Larochelle, H., Ranzato, M., Hadsell, R., Balcan, M.~F., and Lin,
  H., editors, {\em Advances in Neural Information Processing Systems},
  volume~33, pages 4392--4402. Curran Associates, Inc.

\bibitem[Hernandez-Leal et~al., 2017]{hernandez2017survey}
Hernandez-Leal, P., Kaisers, M., Baarslag, T., and de~Cote, E.~M. (2017).
\newblock A survey of learning in multiagent environments: Dealing with
  non-stationarity.
\newblock {\em arXiv preprint arXiv:1707.09183}.

\bibitem[Hochreiter and Schmidhuber, 1997]{hochreiter1997long}
Hochreiter, S. and Schmidhuber, J. (1997).
\newblock Long short-term memory.
\newblock {\em Neural computation}, 9(8):1735--1780.

\bibitem[Juliani et~al., 2018]{juliani2018unity}
Juliani, A., Berges, V.-P., Teng, E., Cohen, A., Harper, J., Elion, C., Goy,
  C., Gao, Y., Henry, H., Mattar, M., et~al. (2018).
\newblock Unity: A general platform for intelligent agents.
\newblock {\em arXiv preprint arXiv:1809.02627}.

\bibitem[Kaelbling et~al., 1996]{Kaelbling96JAIR}
Kaelbling, L.~P., Littman, M., and Moore, A. (1996).
\newblock Reinforcement learning: A survey.
\newblock {\em Journal of AI Research}, 4:237--285.

\bibitem[Kearns and Singh, 2002]{kearns2002near}
Kearns, M. and Singh, S. (2002).
\newblock Near-optimal reinforcement learning in polynomial time.
\newblock {\em Machine learning}, 49(2):209--232.

\bibitem[Konda and Tsitsiklis, 1999]{konda1999actor}
Konda, V. and Tsitsiklis, J. (1999).
\newblock Actor-critic algorithms.
\newblock {\em Advances in neural information processing systems}, 12.

\bibitem[Kumar et~al., 2011]{Kumar11IJCAI}
Kumar, A., Zilberstein, S., and Toussaint, M. (2011).
\newblock Scalable multiagent planning using probabilistic inference.
\newblock In {\em Proc.\ of the International Joint Conference on Artificial
  Intelligence}, pages 2140--2146.

\bibitem[Laurent et~al., 2011]{laurent2011world}
Laurent, G.~J., Matignon, L., Fort-Piat, L., et~al. (2011).
\newblock The world of independent learners is not markovian.
\newblock {\em International Journal of Knowledge-based and Intelligent
  Engineering Systems}, 15(1):55--64.

\bibitem[Li et~al., 2021]{li2021dealing}
Li, W., Wang, X., Jin, B., Sheng, J., and Zha, H. (2021).
\newblock Dealing with non-stationarity in marl via trust-region decomposition.
\newblock In {\em International Conference on Learning Representations}.

\bibitem[Lopez et~al., 2018]{SUMO2018}
Lopez, P.~A., Behrisch, M., Bieker-Walz, L., Erdmann, J., Fl{\"o}tter{\"o}d,
  Y.-P., Hilbrich, R., L{\"u}cken, L., Rummel, J., Wagner, P., and Wie{\ss}ner,
  E. (2018).
\newblock Microscopic traffic simulation using sumo.
\newblock In {\em The 21st IEEE International Conference on Intelligent
  Transportation Systems}. IEEE.

\bibitem[Lowe et~al., 2017]{Lowe17NIPS}
Lowe, R., Wu, Y., Tamar, A., Harb, J., Abbeel, P., and Mordatch, I. (2017).
\newblock Multi-agent actor-critic for mixed cooperative-competitive
  environments.
\newblock {\em Neural Information Processing Systems (NIPS)}.

\bibitem[Lyu et~al., 2021]{lyu2021contrasting}
Lyu, X., Xiao, Y., Daley, B., and Amato, C. (2021).
\newblock Contrasting centralized and decentralized critics in multi-agent
  reinforcement learning.
\newblock {\em Proceedings of the Twentieth International Conference on
  Autonomous Agents and MultiAgent Systems}.

\bibitem[McCallum, 1995]{McCallum95PhD}
McCallum, A.~K. (1995).
\newblock {\em Reinforcement Learning with Selective Perception and Hidden
  State}.
\newblock PhD thesis, University of Rochester.

\bibitem[Nair et~al., 2003]{Nair03IJCAI}
Nair, R., Tambe, M., Yokoo, M., Pynadath, D.~V., and Marsella, S. (2003).
\newblock Taming decentralized {POMDP}s: Towards efficient policy computation
  for multiagent settings.
\newblock In {\em Proc.\ of the International Joint Conference on Artificial
  Intelligence}, pages 705--711.

\bibitem[Nair et~al., 2005]{Nair05AAAI}
Nair, R., Varakantham, P., Tambe, M., and Yokoo, M. (2005).
\newblock Networked distributed {POMDP}s: A synthesis of distributed constraint
  optimization and {POMDP}s.
\newblock In {\em Proc.\ of the National Conference on Artificial
  Intelligence}, pages 133--139.

\bibitem[\noopsort{Pol}van~der Pol and Oliehoek, 2016]{VanDerPol16NIPSWS}
\noopsort{Pol}van~der Pol, E. and Oliehoek, F.~A. (2016).
\newblock Coordinated deep reinforcement learners for traffic light control.
\newblock NIPS'16 Workshop on Learning, Inference and Control of Multi-Agent
  Systems.

\bibitem[Oliehoek et~al., 2021]{oliehoek2021sufficient}
Oliehoek, F., Witwicki, S., and Kaelbling, L. (2021).
\newblock A sufficient statistic for influence in structured multiagent
  environments.
\newblock {\em Journal of Artificial Intelligence Research}, 70:789--870.

\bibitem[Oliehoek and Amato, 2016]{Oliehoek16Book}
Oliehoek, F.~A. and Amato, C. (2016).
\newblock {\em A Concise Introduction to Decentralized POMDPs}.
\newblock Springer Briefs in Intelligent Systems. Springer.

\bibitem[Oliehoek et~al., 2012]{oliehoek2012influence}
Oliehoek, F.~A., Witwicki, S.~J., and Kaelbling, L.~P. (2012).
\newblock Influence-based abstraction for multiagent systems.
\newblock In {\em AAAI12}.

\bibitem[Pearl, 1988]{pearl88}
Pearl, J. (1988).
\newblock {\em Probabilistic Reasoning In Intelligent Systems: Networks of
  Plausible Inference}.
\newblock Morgan Kaufmann.

\bibitem[Rabinowitz et~al., 2018]{rabinowitz2018machine}
Rabinowitz, N., Perbet, F., Song, F., Zhang, C., Eslami, S.~A., and Botvinick,
  M. (2018).
\newblock Machine theory of mind.
\newblock In {\em International conference on machine learning}, pages
  4218--4227. PMLR.

\bibitem[Raileanu et~al., 2018]{raileanu2018modeling}
Raileanu, R., Denton, E., Szlam, A., and Fergus, R. (2018).
\newblock Modeling others using oneself in multi-agent reinforcement learning.
\newblock In {\em International conference on machine learning}, pages
  4257--4266. PMLR.

\bibitem[Schulman et~al., 2017]{schulman2017proximal}
Schulman, J., Wolski, F., Dhariwal, P., Radford, A., and Klimov, O. (2017).
\newblock Proximal policy optimization algorithms.
\newblock {\em arXiv preprint arXiv:1707.06347}.

\bibitem[Spooner et~al., 2021]{spooner2021factored}
Spooner, T., Vadori, N., and Ganesh, S. (2021).
\newblock Factored policy gradients: Leveraging structure for efficient
  learning in {MOMDPs}.
\newblock {\em Advances in Neural Information Processing Systems}, 34.

\bibitem[Suau et~al., 2022a]{suau2022IAM}
Suau, M., He, J., Congeduti, E., Starre, R.~A., Czechowski, A., and Oliehoek,
  F.~A. (2022a).
\newblock Influence-aware memory architectures for deep reinforcement learning
  in pomdps.
\newblock {\em Neural Computing and Applications}, pages 1--17.

\bibitem[Suau et~al., 2022b]{suau2022influence}
Suau, M., He, J., Spaan, M.~T., and Oliehoek, F. (2022b).
\newblock Influence-augmented local simulators: A scalable solution for fast
  deep rl in large networked systems.
\newblock In {\em International Conference on Machine Learning}, pages
  20604--20624. PMLR.

\bibitem[Sutton and Barto, 1998]{SuttonBarto98}
Sutton, R.~S. and Barto, A.~G. (1998).
\newblock {\em Reinforcement Learning: An Introduction}.
\newblock {The MIT Press}.

\bibitem[Tan, 1993]{tan1993multi}
Tan, M. (1993).
\newblock Multi-agent reinforcement learning: Independent vs. cooperative
  agents.
\newblock In {\em Proceedings of the tenth international conference on machine
  learning}, pages 330--337.

\bibitem[Todorov et~al., 2012]{todorov2012mujoco}
Todorov, E., Erez, T., and Tassa, Y. (2012).
\newblock Mujoco: A physics engine for model-based control.
\newblock In {\em 2012 IEEE/RSJ International Conference on Intelligent Robots
  and Systems}, pages 5026--5033. IEEE.

\bibitem[Varakantham et~al., 2007]{Varakantham07AAMAS}
Varakantham, P., Marecki, J., Yabu, Y., Tambe, M., and Yokoo, M. (2007).
\newblock Letting loose a {SPIDER} on a network of {POMDPs}: Generating quality
  guaranteed policies.
\newblock In {\em Proc.\ of the International Conference on Autonomous Agents
  and Multiagent Systems}.

\bibitem[Vinitsky et~al., 2018]{vinitsky2018benchmarks}
Vinitsky, E., Kreidieh, A., Le~Flem, L., Kheterpal, N., Jang, K., Wu, C., Wu,
  F., Liaw, R., Liang, E., and Bayen, A.~M. (2018).
\newblock Benchmarks for reinforcement learning in mixed-autonomy traffic.
\newblock In {\em Conference on robot learning}, pages 399--409. PMLR.

\bibitem[Witwicki and Durfee, 2010]{Witwicki10ICAPS}
Witwicki, S.~J. and Durfee, E.~H. (2010).
\newblock Influence-based policy abstraction for weakly-coupled {D}ec-{POMDP}s.
\newblock In {\em Proc.\ of the International Conference on Automated Planning
  and Scheduling}, pages 185--192.

\bibitem[Witwicki and Durfee, 2011]{Witwicki11AAMAS}
Witwicki, S.~J. and Durfee, E.~H. (2011).
\newblock Towards a unifying characterization for quantifying weak coupling in
  {Dec-POMDPs}.
\newblock In {\em Proceedings of the Tenth International Conference on
  Autonomous Agents and Multiagent Systems}, pages 29--36, Taipei, Taiwan.

\bibitem[Wu et~al., 2017]{wu2017flow}
Wu, C., Kreidieh, A., Parvate, K., Vinitsky, E., and Bayen, A.~M. (2017).
\newblock Flow: A modular learning framework for autonomy in traffic.
\newblock {\em arXiv preprint arXiv:1710.05465}.

\bibitem[Yu et~al., 2021]{yu2021surprising}
Yu, C., Velu, A., Vinitsky, E., Wang, Y., Bayen, A., and Wu, Y. (2021).
\newblock The surprising effectiveness of ppo in cooperative, multi-agent
  games.
\newblock {\em arXiv preprint arXiv:2103.01955}.

\bibitem[Zhang et~al., 2009]{zhang2009integrating}
Zhang, C., Abdallah, S., and Lesser, V. (2009).
\newblock Integrating organizational control into multi-agent learning.
\newblock In {\em Proceedings of The 8th International Conference on Autonomous
  Agents and Multiagent Systems-Volume 2}, pages 757--764.

\bibitem[Zhang et~al., 2010]{zhang2010self}
Zhang, C., Lesser, V.~R., and Abdallah, S. (2010).
\newblock Self-organization for coordinating decentralized reinforcement
  learning.
\newblock In {\em AAMAS}, volume~10, pages 739--746.

\bibitem[Zhu et~al., 2020]{zhu2020robosuite}
Zhu, Y., Wong, J., Mandlekar, A., and Mart{\'\i}n-Mart{\'\i}n, R. (2020).
\newblock robosuite: A modular simulation framework and benchmark for robot
  learning.
\newblock {\em arXiv preprint arXiv:2009.12293}.

\end{thebibliography}

\bibliographystyle{apalike}
\section*{Checklist}


\begin{enumerate}

\item For all authors...
\begin{enumerate}
  \item Do the main claims made in the abstract and introduction accurately reflect the paper's contributions and scope?
    \answerYes{}
  \item Did you describe the limitations of your work?
    \answerYes{}
  \item Did you discuss any potential negative societal impacts of your work?
    \answerNA{}
  \item Have you read the ethics review guidelines and ensured that your paper conforms to them?
    \answerYes{}
\end{enumerate}

\item If you are including theoretical results...
\begin{enumerate}
  \item Did you state the full set of assumptions of all theoretical results?
    \answerYes{}
        \item Did you include complete proofs of all theoretical results?
    \answerYes{All proofs are included in Appendix \ref{ap:proofs}.}
\end{enumerate}

\item If you ran experiments...
\begin{enumerate}
  \item Did you include the code, data, and instructions needed to reproduce the main experimental results (either in the supplemental material or as a URL)?
    \answerYes{Code and instructions on how to reproduce all the experimental results are included in the supplemental material.}
  \item Did you specify all the training details (e.g., data splits, hyperparameters, how they were chosen)?
    \answerYes{See Section \ref{sec:setup} for a description of the experimental setup and Appendix \ref{ap:hyper} for a full list of the hyperparameter configurations used for the experiments along with an explanation of how these were chosen.}
        \item Did you report error bars (e.g., with respect to the random seed after running experiments multiple times)?
    \answerYes{The shaded areas in the learning curves and the black vertical lines on the bar plots in Section \ref{sec:results} and Appendix \ref{ap:results} show the standard error of the mean.}
        \item Did you include the total amount of compute and the type of resources used (e.g., type of GPUs, internal cluster, or cloud provider)?
    \answerYes{The computer requirements and specifications along with a table showing a breakdown of the total runtimes for all experiments are provided in Appendix \ref{ap:runtimes}.}
\end{enumerate}

\item If you are using existing assets (e.g., code, data, models) or curating/releasing new assets...
\begin{enumerate}
  \item If your work uses existing assets, did you cite the creators?
    \answerYes{}
  \item Did you mention the license of the assets?
    \answerYes{}
  \item Did you include any new assets either in the supplemental material or as a URL?
    \answerNo{}
  \item Did you discuss whether and how consent was obtained from people whose data you're using/curating?
    \answerNA{}
  \item Did you discuss whether the data you are using/curating contains personally identifiable information or offensive content?
    \answerNA{}
\end{enumerate}

\item If you used crowdsourcing or conducted research with human subjects...
\begin{enumerate}
  \item Did you include the full text of instructions given to participants and screenshots, if applicable?
    \answerNA{}
  \item Did you describe any potential participant risks, with links to Institutional Review Board (IRB) approvals, if applicable?
    \answerNA{}
  \item Did you include the estimated hourly wage paid to participants and the total amount spent on participant compensation?
    \answerNA{}
\end{enumerate}

\end{enumerate}


\newpage
\appendix

\section{Proofs}\label{ap:proofs}
\onetoone*
\begin{proof}
We will prove it by contradiction. Let us assume there is a single joint policy $\pi$ that induces two different influence distributions $I^1_i$ and $I^2_i$ on agent $i$. From the definition of influence (Section 4.1; \citealt{oliehoek2021sufficient}) we have 
\begin{equation}
    I^1_i(u_i^t|l_i^t) = \sum_{u_{i}^{t-1}, y_{i}^{t-1}, a_{-i}^{t-1}} P^1(u_i^t| x_{i}^{t-1}, u_{i}^{t-1} , y_{i}^{t-1}, a^{t-1}) P^1(u_{i}^{t-1}, y_{i}^{t-1}, a_{-i}^{t-1}|l_i^t)
\end{equation}
and
\begin{equation}
    I^2_i(u_i^t|l_i^t) = \sum_{u_{i}^{t-1}, y_{i}^{t-1}, a_{-i}^{t-1}} P^2(u_i^t| x_{i}^{t-1}, u_{i}^{t-1} , y_{i}^{t-1}, a^{t-1}) P^2(u_{i}^{t-1}, y_{i}^{t-1}, a_{-i}^{t-1}|l_i^t).
\end{equation}
First, we see that, because $\langle x_{i}^{t-1}, u_{i}^{t-1} , y_{i}^{t-1} \rangle$ fully determines the Markov state $s^{t-1}$, the first term in the summation can be computed from the environment's transition function, and thus
\begin{align}
\begin{split}
    P^1(u_i^t| x_{i}^{t-1}, u_{i}^{t-1} , y_{i}^{t-1}, a^{t-1}) =& P^2(u_i^t| x_{i}^{t-1}, u_{i}^{t-1} , y_{i}^{t-1}, a^{t-1}) \\ 
    =&\sum_{s^t} \mathds{1}(u^t, s^t) T(s^t| s^{t-1}, a^{t-1}),
\end{split}
\end{align}
where $\mathds{1}(u^t, s^t)$ is an indicator function that determines if the state $s^t$ is feasible in the context of $u^t$.

Further, we know that 
\begin{equation}
    P^1(u_{i}^{t-1}, y_{i}^{t-1}, a_{-i}^{t-1}|l_i^t) = \sum_{h_{-i}^{t-1}} \pi_{-i}(a_{-i}^{t-1}|h_{-i}^{t-1})P^1(u_{i}^{t-1}y_{i}^{t-1},h_{-i}^{t-1}|l_i^t)
\end{equation}
\begin{equation}
    P^2(u_{i}^{t-1}, y_{i}^{t-1}, a_{-i}^{t-1}|l^t) = \sum_{h_{-i}^{t-1}} \pi_{-i}(a_{-i}^{t-1}|h_{-i}^{t-1})P^2(u_{i}^{t-1}y_{i}^{t-1},h_{-i}^{t-1}|l_i^t)
\end{equation}
where $P^1(u_{i}^{t-1}y_{i}^{t-1},h_{-i}^{t-1}|l^t)$ and $P^2(u_{i}^{t-1}y_{i}^{t-1},h_{-i}^{t-1}|l^t)$ can be computed recursively as
\begin{align}
\begin{split}
    P^1(u_{i}^{t-1},y_{i}^{t-1},h_{-i}^{t-1}|l_i^t) &= \\ \sum_{h_{-i}^{t-2}, o_{-i}^{t-1}}&O(o_{-i}^{t-1}|x_i^{t-1}, u_i^{t-1}, y_i^{t-1})\pi_{-i}(a_{-i}^{t-2}|h_{-i}^{t-2}) P^1(u_{i}^{t-1},y_{i}^{t-1},h_{-i}^{t-2}|l_i^t),
    \label{eq:recursive1}
\end{split}    
\end{align}

and
\begin{align}
\begin{split}
    P^2(u_{i}^{t-1},y_{i}^{t-1},h_{-i}^{t-1}|l_i^t) &= \\ \sum_{h_{-i}^{t-2}, o_{-i}^{t-1}}&O(o_{-i}^{t-1}|x_i^{t-1}, u_i^{t-1}, y_i^{t-1})\pi_{-i}(a_{-i}^{t-2}|h_{-i}^{t-2}) P^2(u_{i}^{t-1},y_{i}^{t-1},h_{-i}^{t-2}|l_i^t),
    \label{eq:recursive2}
\end{split}    
\end{align}
with $h_{-i}^{t-1} = \langle h_{-i, t-2}, a_{-i, t-2}, o_{-i}^{t-1} \rangle$. 
Then, if we further unroll equations \eqref{eq:recursive1} and \eqref{eq:recursive2} up to timestep $0$, we see that all probability distributions in both cases are equivalent and we reach a contradiction. Hence,
\begin{equation}
I^1_i(u_i^t|l_i^t) = I^2_i(u_i^t|l_i^t)    
\end{equation}
\end{proof}

\manytoone*
\begin{proof}
From Proposition \ref{prop:onetoone} it follows that the space of joint influences $\Psi$ is at most as large as the space of joint policies $\Pi$, $|\Psi| \rlap{\kern.45em$|$}> |\Pi|$. Hence, we just need to show that in some cases $\Pi$ is strictly greater than $\Psi$, $|\Pi| > |\Psi|$.
\end{proof}
A clear example is that where each agent's local region $X_i$ is independent of the other agents' policies $\pi_{-i}$ \citep{Becker03AAMAS}. That is, the actions of other agents $a_{-i}$ have no effect on agent $i$'s local state transitions. From the definition of IALM (Definition \ref{def:IALM}) we know that, in our setting, $a_{-i}$ can only affect the local state transitions through $u_i$. Therefore, for the local transitions to be independent the following should hold
\begin{equation}
    P(u_i^t| x_{i}^{t-1}, u_{i}^{t-1} , y_{i}^{t-1}, a^{t-1}) = P(u_i^t| x_{i}^{t-1}, u_{i}^{t-1} , y_{i}^{t-1}, a_i^{t-1})
\end{equation}
The equation above reflects that only agent $i$ can affect $u_i^t$. Thus, in the event of local transition independence, we have that
\begin{equation}
   \forall i \in N:\exists ! I^*_i(u_i^t|l_i^t) \in \Psi_i : \forall \pi \in \Pi \left(I_i(u_i^t|l_i^t, \pi) = I^*_i(u_i^t|l_i^t) \right)
\end{equation}
That is, for any joint policy $\pi \in \Pi$ there is a unique influence distribution $I^*_i \in \Psi_i$ for every agent $i \in N$, and thus, in this particular case, $|\Pi| \gg |\Psi| = 1$.



To prove Lemma \ref{lemma:two} we will use the following lemma.
\begin{lemma}
Let $I_i^1(u_i^t|l_i^t)$ and $I_i^2(u_i^t|l_i^t)$ be two different influence distributions with $M_i^1$ and $M_i^2$ being the IALMs induced by each of them respectively. Moreover, let $P^1(h_i^{t+1}|h_i^t, a_i^t)$ and  $P^2(h_i^{t+1}|h_i^t, a_i^t)$ denote the resulting local AOH transitions for $M_i^1$ and $M_i^2$ respectively. The following inequality holds
\begin{equation}
    \sum_{x_i^{t+1}}\left|P^1(h_i^{t+1}|h_i^t, a_i^t) - P^2(h_i^{t+1}|h_i^t, a_i^t)\right| \leq \sum_{l_i^t,u_i^t} P(l_i^t|h_i^t)  \left|I^1(u_i^t| l_i^t) - I^2(u_i^t| l_i^t) \right| \quad \forall h_i^t, a_i^t
\end{equation}
\label{lemmaone}
\end{lemma}
\begin{proof}
\begin{align}
\begin{split}
&\sum_{h_i^{t+1}}\left|P^1(h_i^{t+1}|h_i^t, a_i^t) - P^2(h_i^{t+1}|h_i^t, a_i^t)\right|\\ 
=& \sum_{o_i^{t+1}} \Big| \sum_{x_i^{t+1}} O_i(o_i^{t+1}|x_i^{t+1})\sum_{u_i^t}\dot{T}_i(x_i^{t+1}| x_i^t, u_i^t, a_i^t)\sum_{l_i^t} I^1(u_i^t|l_i^t)P(l_i^t|h_i^t) \\
&- \sum_{x_i^{t+1}} O_i(o_i^{t+1}|x_i^{t+1})\sum_{u_i^t}\dot{T}_i(x_i^{t+1}| x_i^t, u_i^t, a_i^t)\sum_{l_i^t} I^2(u_i^t|l_i^t)P(l_i^t|h_i^t)\Big|\\
=&  \sum_{o_i^{t+1}}\Big| \sum_{x_i^{t+1}}O_i(o_i^{t+1}|x_i^{t+1})\sum_{u_i^t}\dot{T}_i(x_i^{t+1}| x_i^t, u_i^t, a_i^t) \sum_{l_i^t}P(l_i^t|h_i^t) \big[ I^1(u_i^t|l_i^t) - I^2(u_i^t|l_i^t)\big]\Big|\\
=& \Big|\sum_{l_i^t,u_i^t} P(l_i^t|h_i^t) \big[ I^1(u_i^t|l_i^t) - I^2(u_i^t|l_i^t)\big]\Big|\\
\leq& \sum_{l_i^t,u_i^t} P(l_i^t|h_i^t)  \left|I^1(u_i^t| l_i^t) - I^2(u_i^t| l_i^t) \right|
\end{split}
\end{align}
\end{proof}
\propositiontwo*
\begin{proof}

This is a special case of the simulation lemma \citep{kearns2002near}. We have that the set of local states and actions is the same for both IALMs. Moreover, the reward function is also the same $R^1(x_t, a_t) = R^2(x_t, a_t)$. 


\begin{align}
\begin{split}
    \Big|Q^{\pi_i}_{M_i^1}(h_i^t, a_i^t) - Q^{\pi_i}_{M_i^2}(h_i^t, a_i^t)\Big| =& \Bigg|\sum_{x_i^t} P(x_i^t|h_i^t)  R(x_i^t,  a_i^t)    \\
    +\sum_{h_i^{t+1}, a_i^{t+1}}P^1(h_i^{t+1}|h_i^t, a_i^t)&\pi_i(a_i^{t+1}|h_i^{t+1})  Q^{\pi_i}_{M_i^1}(h_i^{t+1}, a_i^{t+1}) - \sum_{x_i^t} P(x_i^t|h_i^t)  R(x_i^t,  a_i^t)\\
    -\sum_{h_i^{t+1}, a_i^{t+1}}P^2(h_i^{t+1}|h_i^t, a_i^t)&\pi_i(a_i^{t+1}|h_i^{t+1})  Q^{\pi_i}_{M_i^2}(h_i^{t+1}, a_i^{t+1}) \Bigg|, \\
\end{split}
\end{align}
where $P^1(h_i^{t+1}|h_i^t, a_i^t)$ and $P^2(h_i^{t+1}|h_i^t, a_i^t)$ are the AOH transitions induced by $I^1$ and $I^2$ respectively.
\begin{align}
\begin{split}
     \Big|Q^{\pi_i}_{M_i^1}(h_i^t, a_i^t) - Q^{\pi_i}_{M_i^2}(h_i^t, a_i^t)\Big|=& \Bigg|\sum_{h_i^{t+1},a_i^{t+1}}\pi_i(a_i^{t+1}|h_i^{t+1})\Big[ \\
     P^1(h_i^{t+1}|h_i^t, a_i^t) &Q^{\pi_i}_{M_i^1}(h_i^{t+1}, a_i^{t+1}) -  P^2(h_i^{t+1}|h_i^t, a_i^t)Q^{\pi_i}_{M_i^2}(h_i^{t+1}, a_i^{t+1})\Big]\Bigg| \\
    =& \Bigg| \sum_{h_i^{t+1},a_i^{t+1}}\pi_i(a_i^{t+1}|h_i^{t+1})\Big[\\
     P^1(h_i^{t+1}|h_i^t, a_i^t) &Q^{\pi_i}_{M_i^1}(h_i^{t+1}, a_i^{t+1}) -  P^2(h_i^{t+1}|h_i^t,a_i^t) Q^{\pi_i}_{M_i^1}(h_i^{t+1}, a_i^{t+1}) \\ 
     + P^2(h_i^{t+1}|h_i^t, a_i^t)&Q^{\pi_i}_{M_i^1}(h_i^{t+1}, a_i^{t+1}) -  P^2(h_i^{t+1}|h_i^t, a_i^t)Q^{\pi_i}_{M_i^2}(h_i^{t+1}, a_i^{t+1})\Big] \Bigg|\\
     \leq& 
      \Bigg|\bar{R}(H - t)\sum_{h_i^{t+1}}\big(P^1(h_i^{t+1}|h_i^t, a_i^t) -  P^2(h_i^{t+1}|h_i^t,a_i^t)\big) \\ +
     \sum_{h_i^{t+1},a_i^{t+1}}\pi_i(a_i^{t+1}&|h_i^{t+1})P^2(h_i^{t+1}|h_i^t, a_i^t)\Big[Q^{\pi_i}_{M_i^1}(h_i^{t+1}, a_i^{t+1}) -  Q^{\pi_i}_{M_i^2}(h_i^{t+1}, a_i^{t+1})\Big]\Bigg|
\end{split}
\end{align}
since $Q^{\pi_i}_{M_i^1}(h_i^{t+1}) \leq \bar{R}(H - t)$. Then, from Lemma \ref{lemmaone} we know that
\begin{equation}
    \sum_{h_i^{t+1}}\big(P^1(h_i^{t+1}|h_i^t, a_i^t) -  P^2(h_i^{t+1}|h_i^t,a_i^t)\big) \leq \sum_{l_i^t,u_i^t} P(l_i^t|h_i^t)  \left|I^1(u_i^t| l_i^t) - I^2(u_i^t| l_i^t) \right| \leq \xi \quad \forall h_i^t, a_i^t.
\end{equation}
Hence, 
\begin{equation}
    \big|Q^{\pi_i}_{M_i^1}(h_i^t, a_i^t) - Q^{\pi_i}_{M_i^2}(h_i^t, a_i^t)\big| \leq  \sum_{k=t}^H\bar{R}(H - k)\xi = \bar{R}\frac{(H-t)(H-t+1)}{2}\xi.
\end{equation}
\end{proof}
\theoremtwo*
\begin{proof}
We will prove it by contradiction. Let us assume there is a policy $\pi^*$ that is optimal for $M_i^1$ but not for $M_i^2$. 
This implies that, for some $h^t_i$, there is at least one action  $\hat{a}^t_i \neq \bar{a}_i^t$ for which
\begin{equation}
    Q^{\pi^*}_{M_i^2}(h_i^t, \bar{a}_i^t) < Q^{\pi^*}_{M_i^2}(h_i^t, \hat{a}_i^t)
\end{equation}
Then, because the maximum gap between $Q_{M_i^1}$ and $Q_{M_i^2}$ is $\Delta$,
\begin{equation}
      Q^{\pi^*}_{M_i^1}(h_i^t, \bar{a}_i^t) - \Delta \leq Q^{\pi^*}_{M_i^2}(h_i^t, \bar{a}_i^t) < Q^{\pi^*}_{M_i^2}(h_i^t, \hat{a}_i^t) \leq Q^{\pi^*}_{M_i^1}(h_i^t, \hat{a}_i^t) + \Delta.
\end{equation}
Therefore, we have
\begin{equation}
      Q^{\pi^*}_{M_i^1}(h_i^t, \bar{a}_i^t) - Q^{\pi^*}_{M_i^1}(h_i^t, \hat{a}_i^t) < 2\Delta,
\end{equation}
which contradicts the statement
\begin{equation}
    Q^{\pi^*}_{M_i^1}(h_i^t, \bar{a}_i^t) - Q^{\pi^*}_{M_i^1}(h_i^{t}, {a_i^t}) > 2\Delta \quad \forall h_i^t, a_i^t
\end{equation}
\end{proof}
\section{Further Related Work}\label{ap:related_work}


There is a sizeable body of literature that concentrates on the non-stationarity issues arising from having multiple agents learning simultaneously in the same environment \citep{laurent2011world, hernandez2017survey}. Although oftentimes the problem can be simply ignored with virtually no consequences for the agents’ performance \citep{tan1993multi}, in general, disregarding changes in the other agents' policies, and assuming individual Q-values to be stationary, can have a catastrophic effect on convergence \citep{Claus98AAAI}.  


The problem of non-stationarity becomes even more severe in the Dec-POMDP setting \citep{Oliehoek16Book} since policy changes may not be immediately evident from each agent's AOH. To compensate for this
\citet{raileanu2018modeling} and \citet{rabinowitz2018machine} explicitly train models that predict the other agents' goals and behaviors. In contrast, \citet{Foerster18AAMAS_LOLA} add an extra term to the learning objective that is meant to predict the other agents' parameter updates. This approach is empirically shown to encourage cooperation in general-sum games. In order to better approximate the value function, several works have studied the use of additional information during training to inform each individual agent of changes in the other agents' policies, leading to the ubiquitous centralized training decentralized execution (CTDE) paradigm. The works by \citet{Lowe17NIPS} and \citet{Foerster18AAAI} exploit this by training a single centralized critic that takes as input the true state and joint action of all the agents. This critic is then used to update the policies of all agents following the actor-critic policy gradient update \citep{konda1999actor}. Even though the use of additional information to augment the critic may help reduce bias in the value estimates, the idea lacks any theoretical guarantees and has been shown to produce the same policy gradient in expectation as those produced by multiple independent critics \citep{lyu2021contrasting}. Moreover, according to \citeauthor{lyu2021contrasting}, naively augmenting the critic with all other agents' actions and observations can heavily increase the variance of the policy gradients. Both results, however, assume that the critics have converged to the true on-policy value estimates. The authors do admit that, in practice, critics are often used even when they have not yet converged. In such situations, centralized critics might provide more stable policy updates since they are better equipped to follow the true non-stationary Q values. Following a similar perspective, the concurrent work by \citet{spooner2021factored} tries to reduce variance by using a per-agent baseline function that removes from the policy gradient the contributions to the joint value estimates of those agents that are conditionally independent, thus effectively providing the agent with more stable updates. The works by \citet{de2020independent} and \citet{yu2021surprising} show that the vanilla PPO algorithm \citep{schulman2017proximal} works already quite well on several multi-agent tasks. \citeauthor{yu2021surprising} attribute the positive empirical results to the clipping parameter $\epsilon$, which prevents individual policies from changing drastically, and in turn, reduces the problem of non-stationarity. \citet{li2021dealing} further analyze this idea and propose a method to estimate the joint policy divergence, which is then used as a constraint in the optimization objective. 



\section{Algorithms}\label{ap:algorithms}
The two algorithms below describe how to generate the datasets  $\{D_i\}_{i \in N}$ with the GS (Algorithm \ref{alg:collect}) and how to simulate trajectories with each of the IALS (Algorithm \ref{alg:sample}).

\begin{algorithm}[h]
\caption{Collect datasets $\{D_i\}_{i \in N}$ with GS}
\begin{algorithmic}[h]
\State \textbf{Input:} $T$, $\{\dot{O}_i\}_{i \in N}$, $\pi^0 = \{\pi^0_i\}_{i \in N}$
\Comment{Global simulator, observation functions, and joint policy}
\For{$n \in \langle 0, ..., N / T \rangle$}
\State  $s^0 \gets$ reset
\Comment{Reset initial state}
\State $\{x^0_i\}_{i \in N} \gets s^0$ 
\Comment{Extract local states from global state}
\State $\{l^0_i \gets x^0_i\}_{i \in N}$
\Comment{Initialize each agent's ALSH with initial local state}
\State $\{o_i^{0} \sim O_i(\cdot \mid x^{0})\}_{i \in N}$
\Comment{Sample each agent's observation from $O_i$}
\State $\{h^0_i \gets o_i^0\}_{i \in N}$
\Comment{Initialize each agent's AOH with initial observation}
\For{$t \in \langle0, ..., T\rangle$}
\State $\{u^0_i\}_{i \in N} \gets s^0$
\Comment{Extract each agent's influence sources from global state}
\State $\{D_i \gets (l^t_i, u^t_i)\}_{i \in N}$ 
\Comment{Append ALSH-influence-source pair to the datasets}
\State $\{a_i^t \sim \pi(\cdot \mid h_i^t)\}_{i \in N}$ 
 \Comment{Sample each agent's action from $\pi_i$}
\State $s^{t+1} \sim T(\cdot \mid s^t, a^t = \{a_i^t\}_{i \in N})$  
\Comment{Sample next state from GS}
\State $\{x^{t+1}_i\}_{i \in N} \gets s^{t+1}$
\Comment{Extract local states from global state}
\State $\{l^{t+1}_i \gets \langle a_i^t, x_i^{t+1} \rangle\}_{i \in N}$ 
\Comment{Append action-local-state pairs to each agent's ALSH}
\State $\{o_i^{t+1} \sim \dot{O}_i(\cdot \mid x^{t+1})\}_{i \in N}$
\Comment{Sample each agent's observation from $\dot{O}_i$}
\State $\{h^{t+1}_i \gets \langle a_i^t, o_i^{t+1}\rangle \}_{i \in N}$
 \Comment{Append actions-observation pairs to each agent's AOH}
\EndFor
\EndFor
\end{algorithmic}
\label{alg:collect}
\end{algorithm}
\begin{algorithm}
\caption{Simulate agent $i$'s trajectory with IALS}
\begin{algorithmic}[1]
\State \textbf{Input:} $\dot{T}_i,\dot{R}_i, \dot{O}_i, \pi_i,\hat{I}_{\theta_i}$
\Comment{local simulator, local reward and observation functions, policy, AIP}
  \State $x^0_i \gets$ reset
  \Comment{Reset initial state}
  \State $o_i^0 \sim \dot{O}_i(\cdot|x_i^0)$  
  \Comment{Sample observation from $\dot{O}_i$}
  \State $h^0_i \gets o^0_i$
  \Comment{Initialize AOH with initial observation}
  \For{$t \in \langle0, ..., T\rangle$}
  \State $a_i^t \sim \pi(\cdot \mid h_i^t)$ 
  \Comment{Sample action}
  \State $\dot{R}_i(x_i^t, a_i^t)$
  \Comment{Compute reward}
  \State  $u_i^t \sim \hat{I}_{\theta_i}(\cdot \mid l_i^t) $ 
  \Comment{Sample influence sources from AIP}
  \State $x_i^{t+1} \sim \dot{T}(\cdot \mid x_i^t, a_i^t, u_i^t)$ 
  \Comment{Sample next local state from LS}
  \State $l_i^{t+1} \gets \langle a_i^t, x_i^{t+1} \rangle $ 
  \Comment{Append action-local-state pair to ALSH}
  \State $o_i^{t+1} \sim \dot{O}_i(\cdot \mid x_i^{t+1})$  
  \Comment{Sample observation from $O$}
  \State $h_i^{t+1} \gets \langle a_i^t, o_i^{t+1} \rangle $ 
  \Comment{Append action-observation pair to AOH}
  \EndFor
\end{algorithmic}
\label{alg:sample}
\end{algorithm}

\section{Results}\label{ap:results}

\subsection{DIALS vs GS}
\begin{figure}[h!]
\vspace{-10pt}
     \centering
     \begin{subfigure}{0.49\textwidth}
         \centering
         \includegraphics[width=\textwidth]{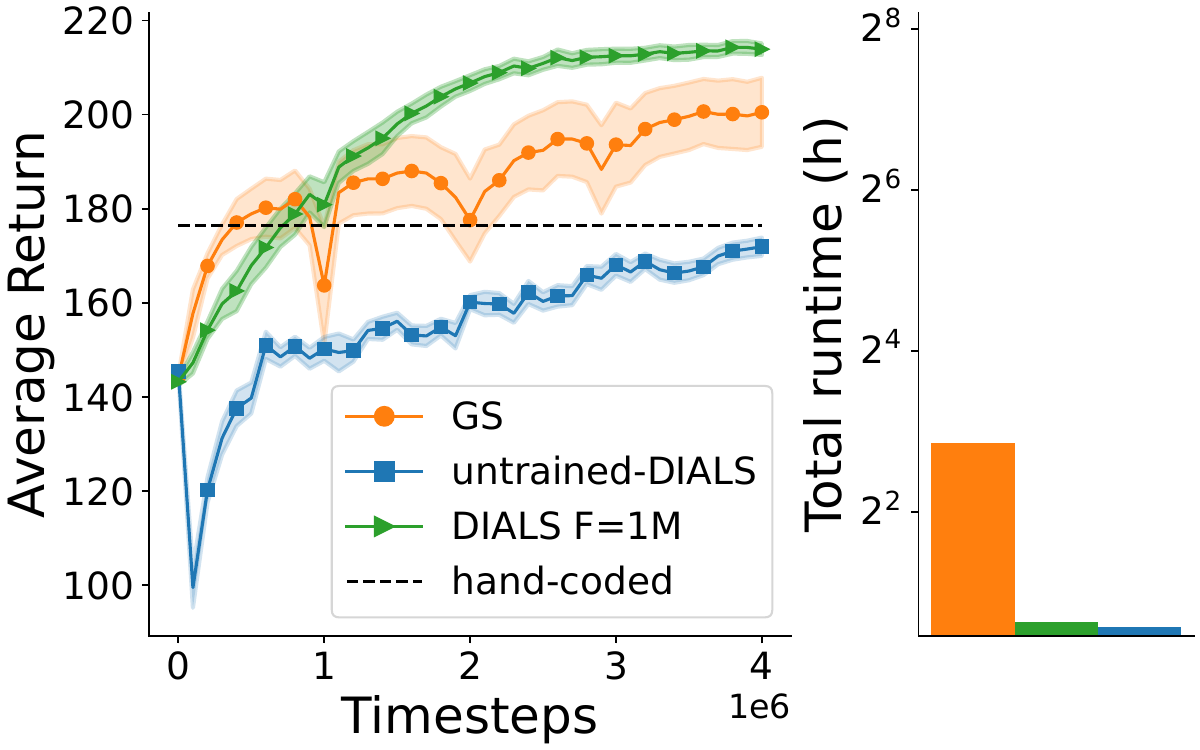}
         \caption{traffic $4$ agents}
         \label{fig:traffic_4agents}
     \end{subfigure}
     \hfill
     \begin{subfigure}{0.49\textwidth}
         \centering
         \includegraphics[width=\textwidth]{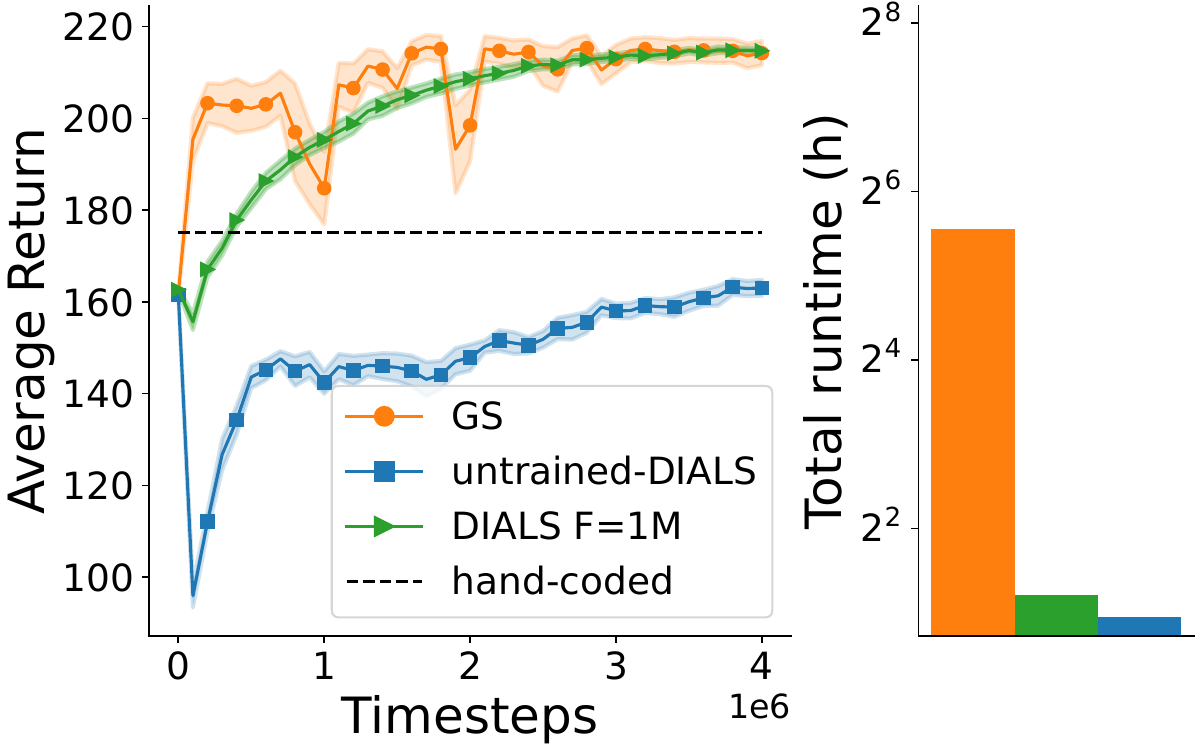}
         \caption{traffic $25$ agents}
     \end{subfigure}
     \hfill
     \begin{subfigure}{0.49\textwidth}
         \centering
         \includegraphics[width=\textwidth]{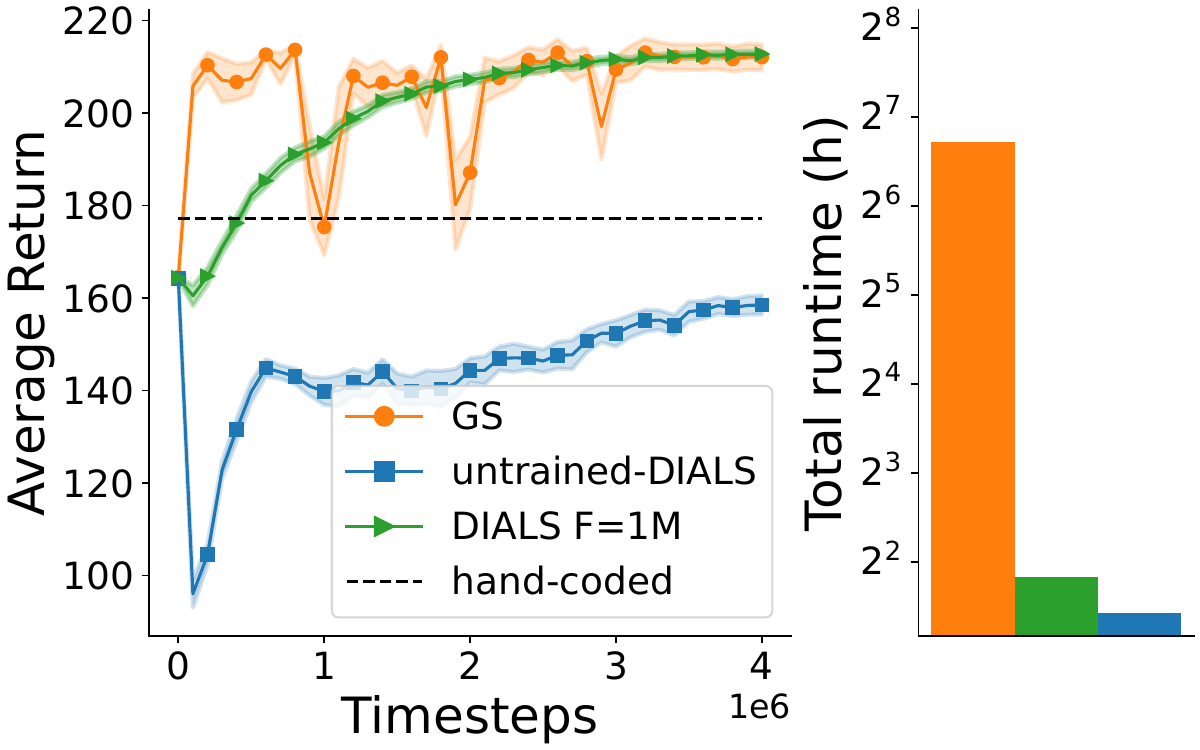}
         \caption{traffic $49$ agents}
     \end{subfigure}
     \hfill
     \begin{subfigure}{0.49\textwidth}
         \centering
         \includegraphics[width=\textwidth]{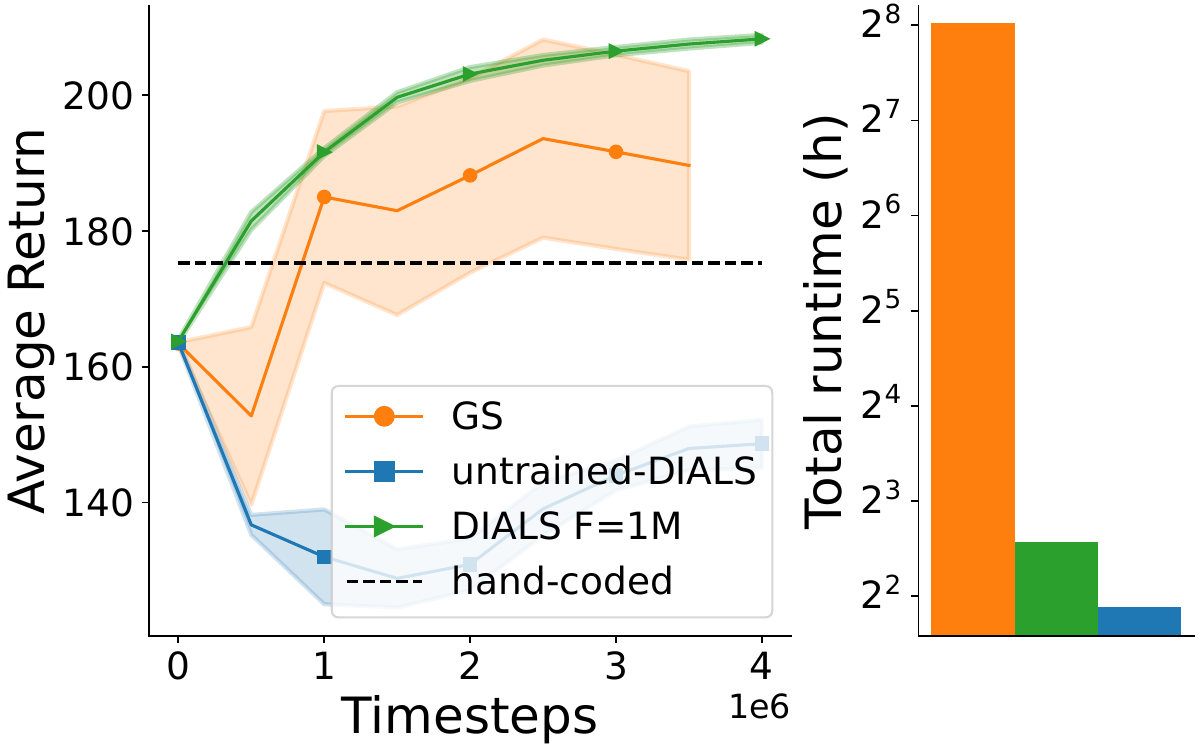}
         \caption{traffic $100$ agents}
         \label{fig:traffic_100agents}
     \end{subfigure}
     \caption{\textbf{Left (a), (b), (c), and (d):} Average return as a function of the number of timesteps with GS, DIALS $F=1$M, and untrained-DIALS on the traffic environment. \textbf{Right (a), (b), (c), and (d):} Total runtime of training for 4M timesteps, $y$-axis is in $\log_2$ scale.}
    \label{fig:traffic_appendix}
\end{figure}
\begin{figure}[h!]
     \centering
     \begin{subfigure}{0.49\textwidth}
         \centering
         \includegraphics[width=\textwidth]{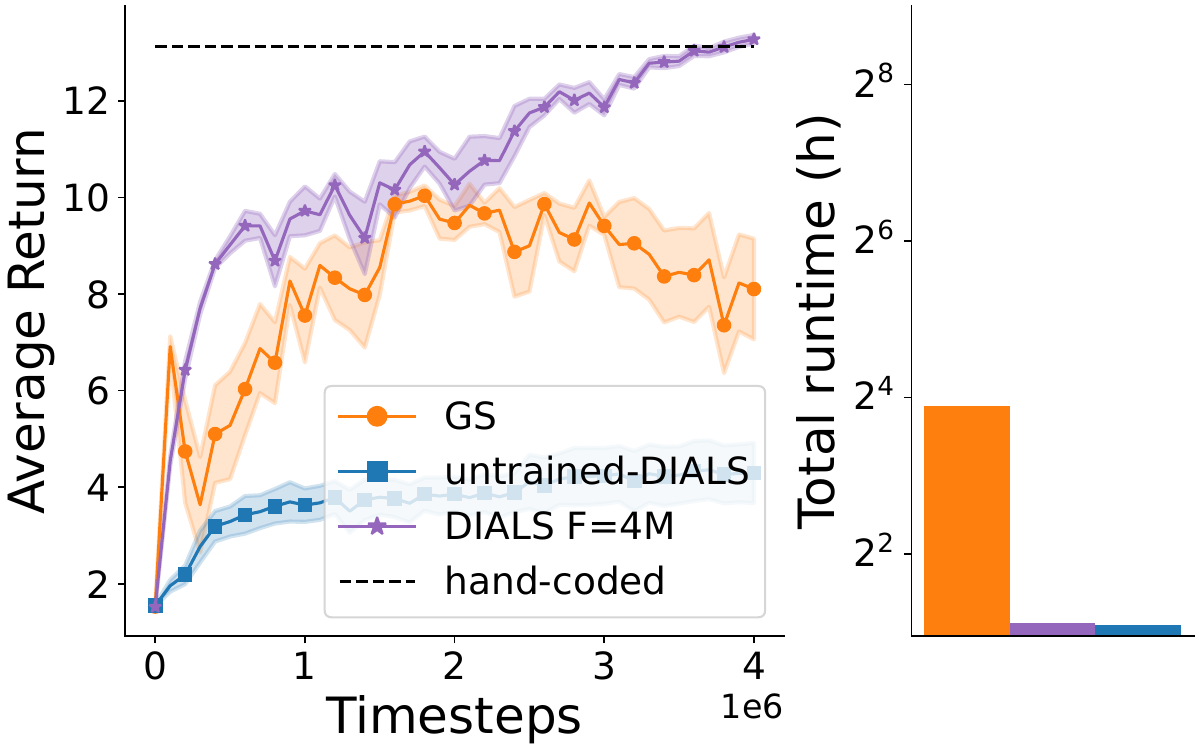}
         \caption{Warehouse $4$ agents}
     \end{subfigure}
     \hfill
     \begin{subfigure}{0.49\textwidth}
         \centering
         \includegraphics[width=\textwidth]{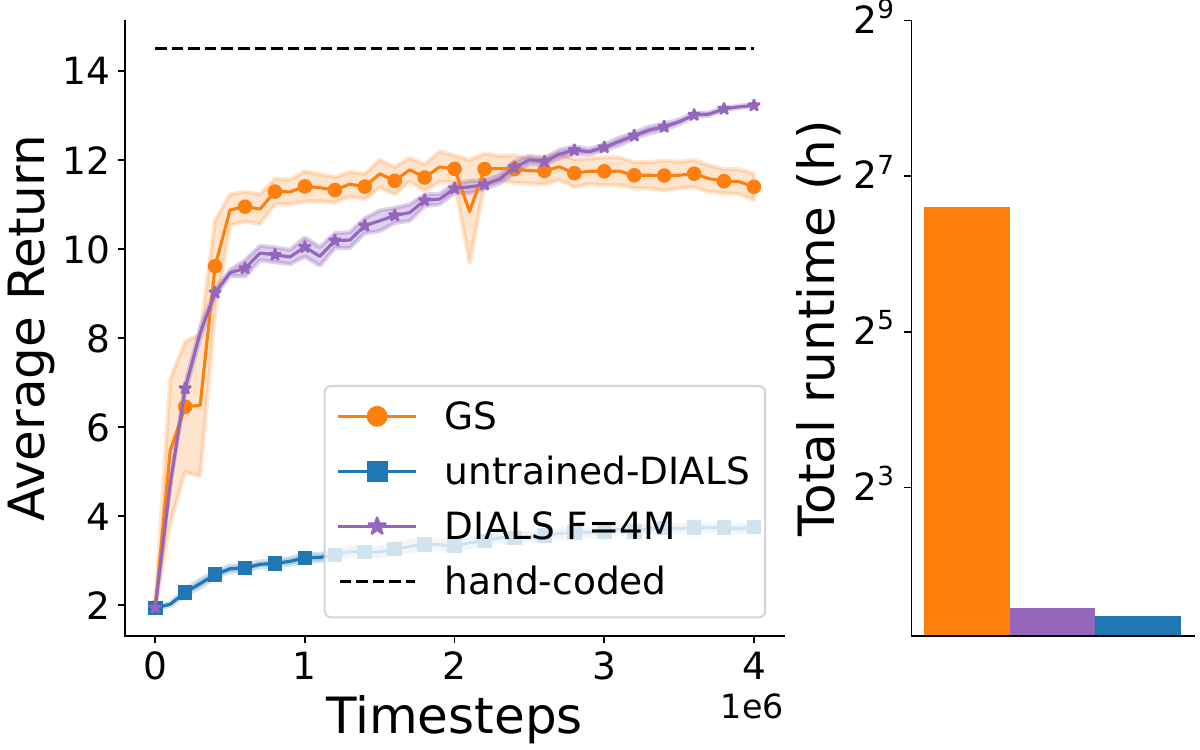}
         \caption{Warehouse $25$ agents}
     \end{subfigure}
     \hfill
     \begin{subfigure}{0.49\textwidth}
         \centering
         \includegraphics[width=\textwidth]{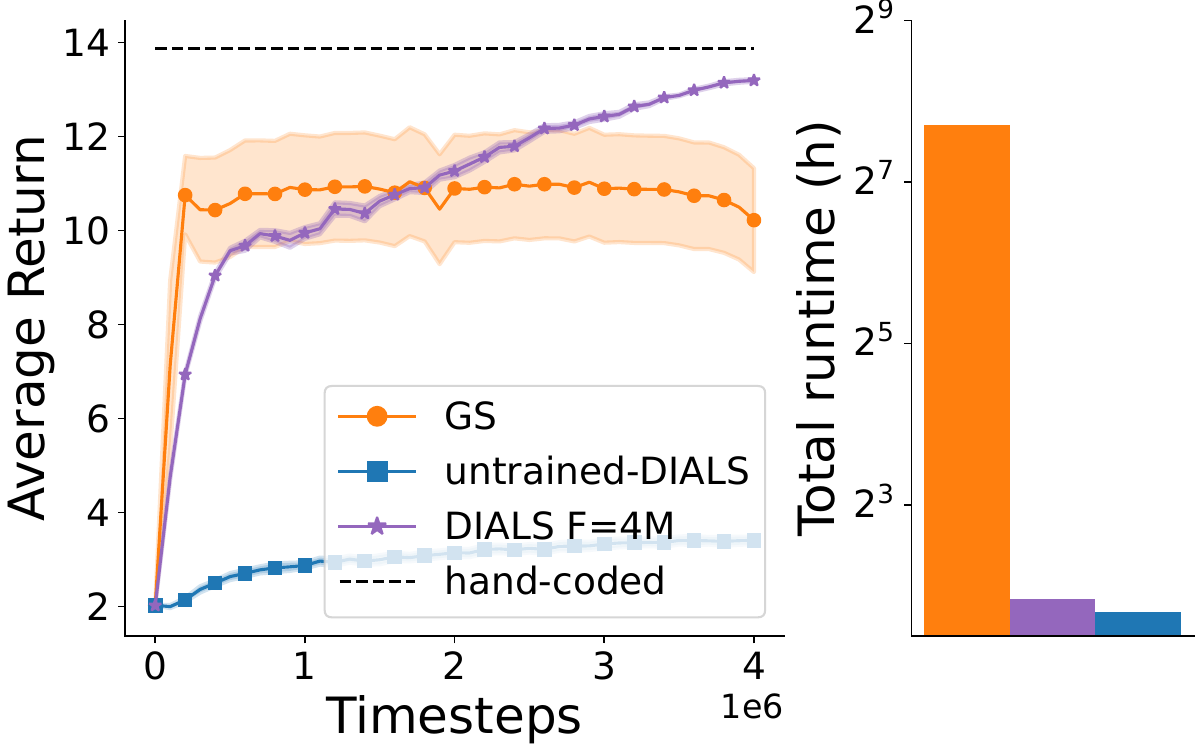}
         \caption{Warehouse $49$ agents}
     \end{subfigure}
     \hfill
     \begin{subfigure}{0.49\textwidth}
         \centering
         \includegraphics[width=\textwidth]{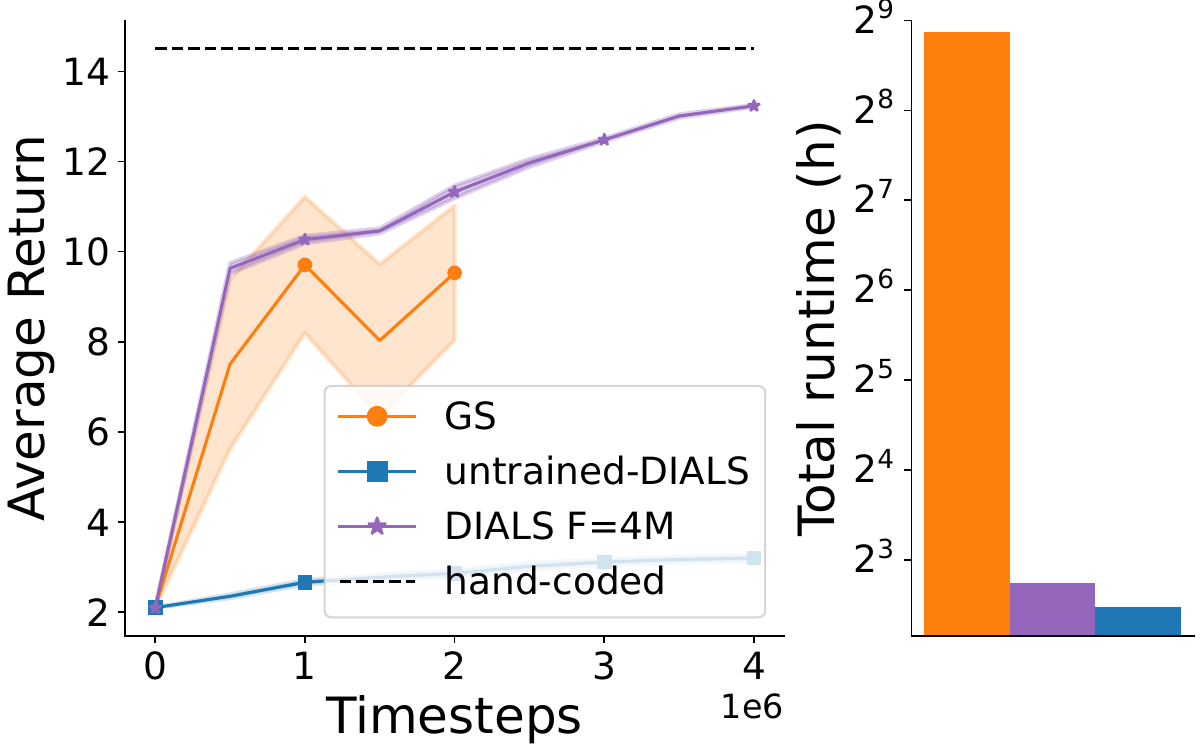}
         \caption{Warehouse $100$ agents}
         \label{fig:warehouse_100agents}
     \end{subfigure}
     \caption{\textbf{Left (a), (b), (c), and (d):} Average return as a function of the number of timesteps with GS, DIALS $F=1$M, and untrained-DIALS on the warehouse environment. \textbf{Right (a), (b), (c), and (d):} Total runtime of training for 4M timesteps, $y$-axis is in $\log_2$ scale.}
     \label{fig:warehouse_appendix}
\end{figure}

The plots in Figures \ref{fig:traffic_appendix} and \ref{fig:warehouse_appendix} show the learning curves of agents trained with the GS, DIALS, and untrained-DIALS on the 4 variants of the traffic and warehouse environments (4, 25, 49, and 100 agents). The bar plots show the total runtime of training for 4M timesteps with the three simulators. Shaded areas indicate the standard error of the mean. 

The orange curves in Figures  \ref{fig:traffic_100agents} and \ref{fig:warehouse_100agents} stop at 3.5M and 2M timesteps, respectively. This is because the maximum execution time allowed by our computer cluster is 1 week, and training 100 agents with the GS takes longer. A breakdown of the runtimes for the three simulators is provided in Appendix \ref{ap:runtimes}. Note that the runtime measurements were made on the only machine in our computer cluster with more than 100 CPUs. This is so that it would fit DIALS when training on the 100-agent variants. However, the experiments that required less than 100 CPUs were ran on different machines with different CPUs.

The bar plots indicate that DIALS is computationally more efficient and scales much better than GS. Note that the $y$ axis is in $\log_2$ scale. Moreover, agents trained with DIALS seem to converge steadily towards similar high-performing policies in both environments, while agents trained with the GS suffer frequent performance drops and often get stuck in local minima. This is evidenced by the oscillations in the orange curves, the poor mean episodic reward, and large standard errors compared to the green (traffic) and purple (warehouse) curves. The plots also reveal that estimating the influence distributions correctly is important, as indicated by the large gap between DIALS and untrained-DIALS in both environments.  

It is worth noting that the gap between GS and DIALS is larger in the warehouse (Figure \ref{fig:warehouse_appendix}) than in the traffic environment (Figure \ref{fig:traffic_appendix}). We posit that this is because, in the warehouse environment, agents are more strongly coupled. To see this imagine that, by random chance during training, a robot starts favoring items from one shelf over the three others. The robot's neighbors might exploit this and start collecting items from the unattended shelves. However, as soon as this first robot changes its policy and starts collecting items more evenly from all four shelves, the neighbor robots will experience a sudden drop in the value of their policies, which can have catastrophic effects on the learning dynamics. With the DIALS, however, agents are trained on separate simulators and only become aware of changes in the joint policy when the AIPs are retrained. This prevents them from constantly co-adapting to one another. This is in line with our discussion in Section \ref{sec:mitigating}. 

\subsection{AIPs training frequency}

The two plots on the left of Figures \ref{fig:traffic_DIALS_appendix} and \ref{fig:warehouse_DIALS_appendix} show a comparison of the agents' average return as a function of runtime for different values of the AIPs training frequency parameter $F$ ($100$K, $500$K, $1$M, and $4$M timesteps). For ease of visualization, since DIALS  $F=500$K, $F=1$M, and $F=4$M take shorter to finish than DIALS $F=100$K, the red, green, and purple curves are extended by dotted horizontal lines. Due to computational limitations, we ran these experiments only on the 4, 25, and 49-agent variants of the two environments. We then chose the best-performing values for $F$ ($F=1$M for traffic and $F=4$M for warehouse) and used those to run DIALS on the environments with 100 agents.

In the traffic domain, the gap between the green and the purple curve (Figure \ref{fig:traffic_DIALS_appendix}) suggests that it is important to retrain the AIPs at least every $1$M timesteps, such that agents become aware of changes in the other agents' policies. This is consistent on all the three variants (Figures \ref{fig:DIALS_traffic_4agents}, \ref{fig:DIALS_traffic_25agents}, and \ref{fig:DIALS_traffic_49agents}). In contrast, in the warehouse domain (Figure \ref{fig:warehouse_DIALS_appendix}), we see that training the AIPs only once at the beginning (DIALS $F=4$M) is sufficient (Figures \ref{fig:DIALS_warehouse_4agents}, \ref{fig:DIALS_warehouse_25agents}, and \ref{fig:DIALS_warehouse_49agents}). In fact, as indicated by the gap between the brown and the rest of the curves, updating the AIPs too frequently (DIALS $F=100$K), aside from increasing the runtimes, seems detrimental to the agents' performance. This is consistent with our hypothesis in Section \ref{sec:mitigating}: ``by not updating the AIPs too frequently, we get a biased but otherwise more consistent learning signal that the agents can rely on to improve their policies.''

The plots on the right of Figures \ref{fig:traffic_DIALS_appendix} and \ref{fig:warehouse_DIALS_appendix} show the average cross-entropy (CE) loss of the AIPs  evaluated on trajectories sampled from the GS. As explained in Section \ref{sec:DIALS} since all agents learn simultaneously, the influence distributions $\{I(u_i^t |l_i^t)\}_{i \in N}$ are non-stationary. For this reason, we see that the CE loss changes as the policies of the other agents are updated. We can also see how the CE loss decreases when the AIPs are retrained, which happens more or less frequently depending on the hyperparameter $F$. Note that the CE not only measures the distance between the two probability distributions but also the absolute entropy. In the warehouse domain (Figure \ref{fig:warehouse_DIALS_appendix}), the neighbor robots' locations become more predictable (lower entropy) as their policies improve. This explains why the CE loss decreases even though the AIPs are not updated. Also note that, in the warehouse environment (Figure \ref{fig:warehouse_DIALS_appendix}), even though by the end of training DIALS $F=4$M is highly inaccurate, as evidenced by the gap between the purple and the other curves, it is still good enough to train policies that match the performance of those trained with DIALS $F=500$K and $F=1$M. This is in line with our results in Section \ref{sec:DIALS}: ``Multiple influence distributions may induce the same optimal policy.''
\begin{figure}
     \centering
     \begin{subfigure}{0.7\textwidth}
         \centering
         \includegraphics[width=\textwidth]{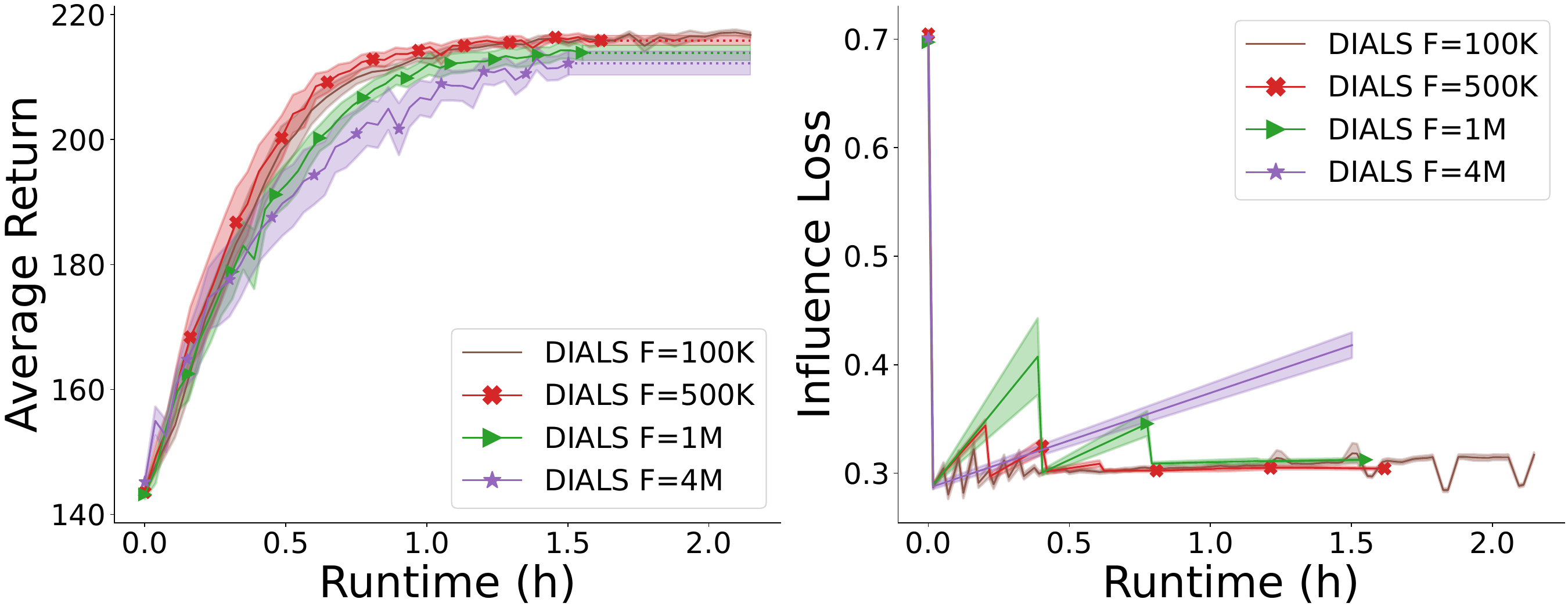}
         \caption{Traffic $4$ agents}
         \label{fig:DIALS_traffic_4agents}
     \end{subfigure}
     \hfill
     \begin{subfigure}{0.7\textwidth}
         \centering
         \includegraphics[width=\textwidth]{figures/DIALS_comparison_5x5.pdf}
         \caption{Traffic $25$ agents}
         \label{fig:DIALS_traffic_25agents}
     \end{subfigure}
     \hfill
     \begin{subfigure}{0.7\textwidth}
         \centering
         \includegraphics[width=\textwidth]{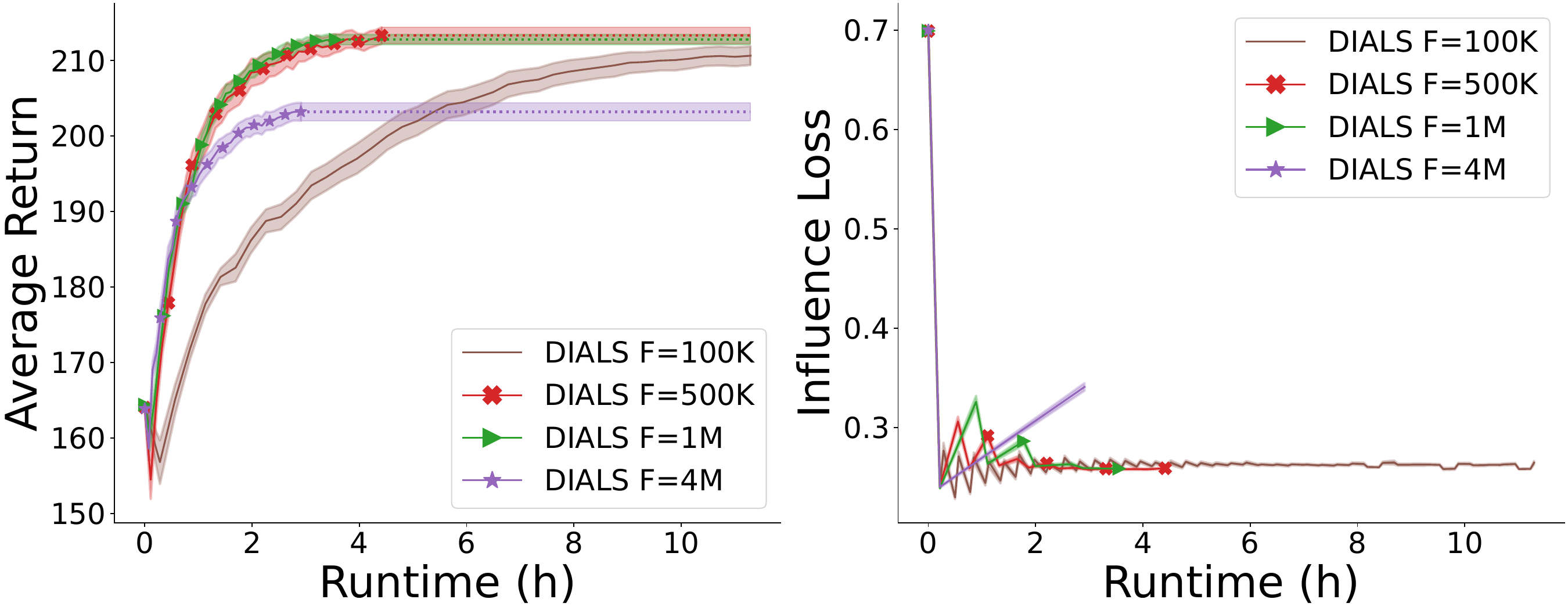}
         \caption{Traffic $49$ agents}
         \label{fig:DIALS_traffic_49agents}
     \end{subfigure}
     \caption{\textbf{Left (a), (b), and (c):} Learning curves for different values of $F$ on the 4, 25, and 49 agent versions of the traffic environment. \textbf{Right (a), (b), and (c):} CE loss of the AIPs as a function of runtime.}
     \label{fig:traffic_DIALS_appendix}
\end{figure}

\begin{figure}
     \centering
     \begin{subfigure}{0.7\textwidth}
         \centering
         \includegraphics[width=\textwidth]{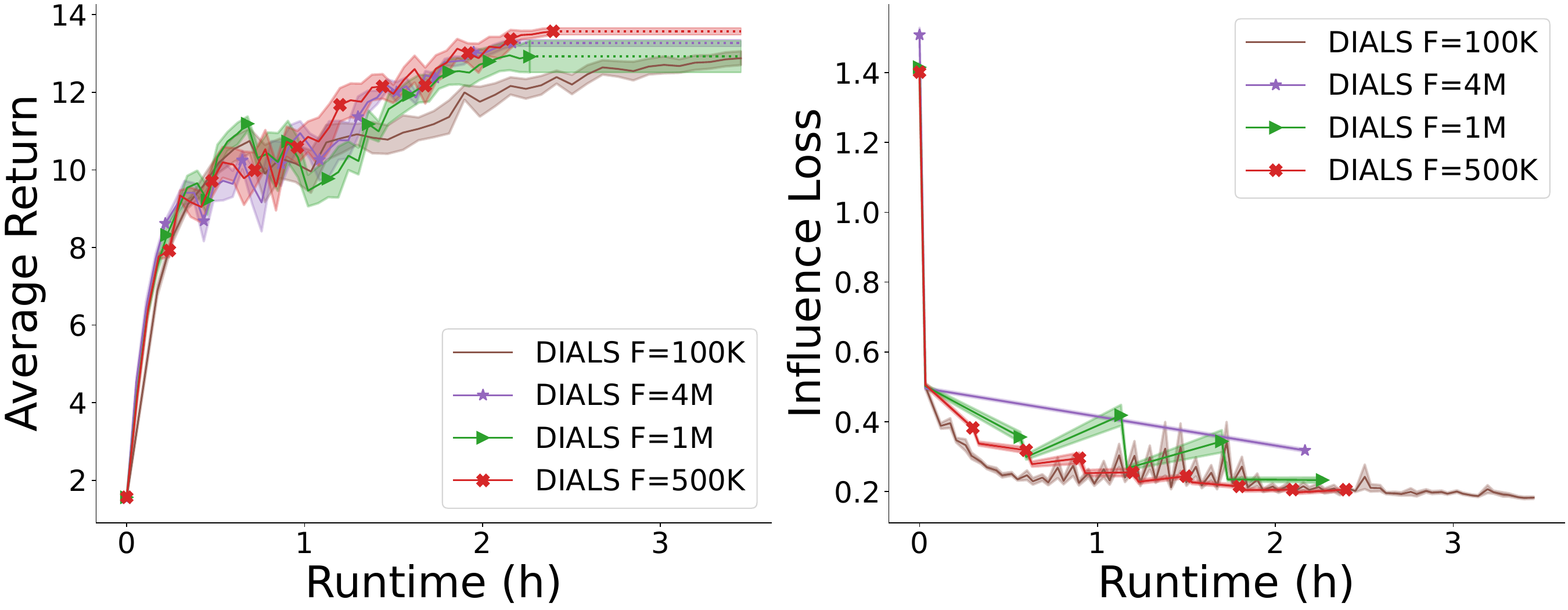}
         \caption{Warehouse $4$ agents}
         \label{fig:DIALS_warehouse_4agents}
     \end{subfigure}
     \hfill
     \begin{subfigure}{0.7\textwidth}
         \centering
         \includegraphics[width=\textwidth]{figures/DIALS_comparison_warehouse_5x5.pdf}
         \caption{Warehouse $25$ agents}
         \label{fig:DIALS_warehouse_25agents}
     \end{subfigure}
     \hfill
     \begin{subfigure}{0.7\textwidth}
         \centering
         \includegraphics[width=\textwidth]{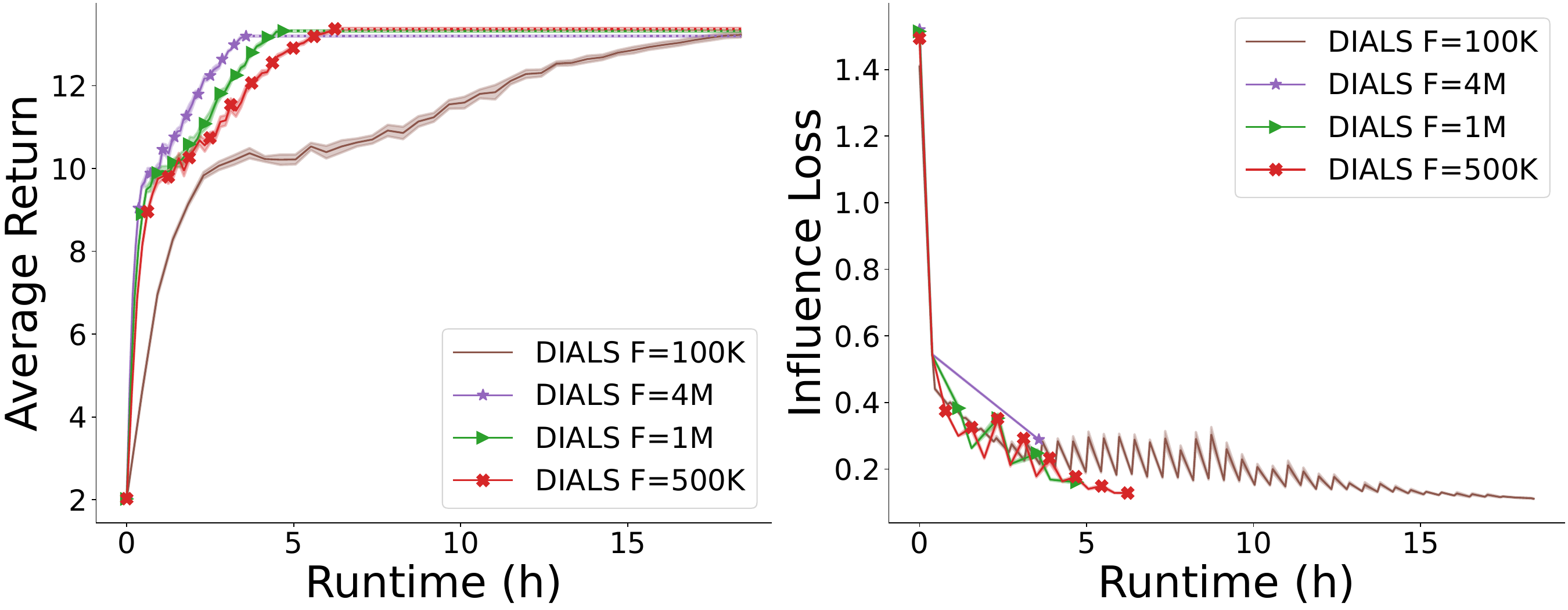}
         \caption{Warehouse $49$ agents}
         \label{fig:DIALS_warehouse_49agents}
     \end{subfigure}
     \caption{\textbf{Left (a), (b), (c), and (d):} Average return as a function of the number of timesteps with GS, DIALS $F=1$M, and untrained-DIALS on the warehouse environment. \textbf{Right (a), (b), (c), and (d):} Total runtime of training for 4M timesteps, $y$-axis is in $\log_2$ scale.}
     \label{fig:warehouse_DIALS_appendix}
\end{figure}
\newpage

\section{Implementation Details}

\subsection{Approximate Influence Predictors}
Due to the sequential nature of the problem, rather than feeding the full past history every time we make a prediction, we use a recurrent neural network (RNN) \citep{hochreiter1997long, Cho2014Learning} and  process observations one at a time,
\begin{equation}
\small
    P(u_t|l_t) \approx \hat{I}_\theta(u_t| \hat{h}_{t-1}, o_t) = F_{\text{rnn}}(\hat{h}_{t-1},o_t, u_t),
    \label{eq:RNN}
\end{equation}
where we use $\hat{h}$ to indicate that the history $h$ is embedded in the RNN's internal memory. 

Given that we generally have multiple influence sources $u_t = \langle u^1_t \dots u^M_t\rangle$, we need to fit $M$ separate models $\hat{I}_{\theta_m}$ to predict each of the $M$ influence sources. In practice, to reduce the computational cost, we can have a single network with a common representation module for all influence sources and output their probability distributions using $M$ separate heads. This representation assumes that the influence sources are independent of one another,
\begin{equation}
\small
    I(u_t|l_t) = \prod_{m=0}^M P(u^m_t|l_t),
\end{equation}
which is true for the two domains we study in this paper.

Finally, although according to the POMDP framework we should condition the AIPs on the full AOH, in many domains, one can exploit the structure of the transitions function to find a subset of variables in the AOH that is sufficient to predict the next observation. This subset is known as the d-separating set \citep{oliehoek2021sufficient}, and as shown in \citet{suau2022IAM} conditioning the AIPs on this rather than the full AOH can ease the task of approximating the influence distribution.

\subsection{Local regions}

When choosing the local regions to build the simulators, the only restriction in terms of size is that these should contain all the necessary information to compute local observations and rewards. In our experiments, we use one simulator per agent since, given that the simulators run in parallel, this is the most computationally efficient way of factorizing the environment. Yet, in certain applications, due to hardware limitations (e.g. not enough CPUs or memory available), it might be necessary to partition the environment into fewer local regions than the number of agents in the environment. 
Moreover, in some environments (including the two we explore here) better results may be obtained by grouping some of them together in the same simulator. In fact, one could potentially treat the agents in the same group/simulator as a single agent and train a policy to control all of them simultaneously. Note, however, that this is orthogonal to our work as we are mainly concerned with computational speedups. 




\section{Simulators}\label{ap:screenshots}
Figure \ref{fig:global_sim} shows two screenshots of the global simulator (GS) for the traffic (left) and warehouse (right) environments with 25 agents each. Figure \ref{fig:local_sim} shows two screenshots of the local simulator (LS) for the traffic (left) and warehouse environments (right). Since all local regions are the same (i.e. $\dot{T}_i$, $\dot{R}_i$, and $\dot{O}_i$ do not change) in the two environments, we use the same LS for all of them. However, because depending on where these are located they are influenced differently by the rest of the system, we train separate AIPs, $\{\hat{I}_{\theta_i}\}_{i \in N}$, for each of them. Note that, we chose the local regions to be the same for simplicity. However, the method can readily be applied to environments with different local transition dynamics $\dot{T}_i$, different local observations $\dot{O}_i$, and/or different local rewards $\dot{R}_i$ for every agent $i \in N$.
\begin{figure}[h!]
\centering
  \begin{subfigure}{0.45\textwidth}
  \hspace{8mm}
    \includegraphics[width=0.795\textwidth]{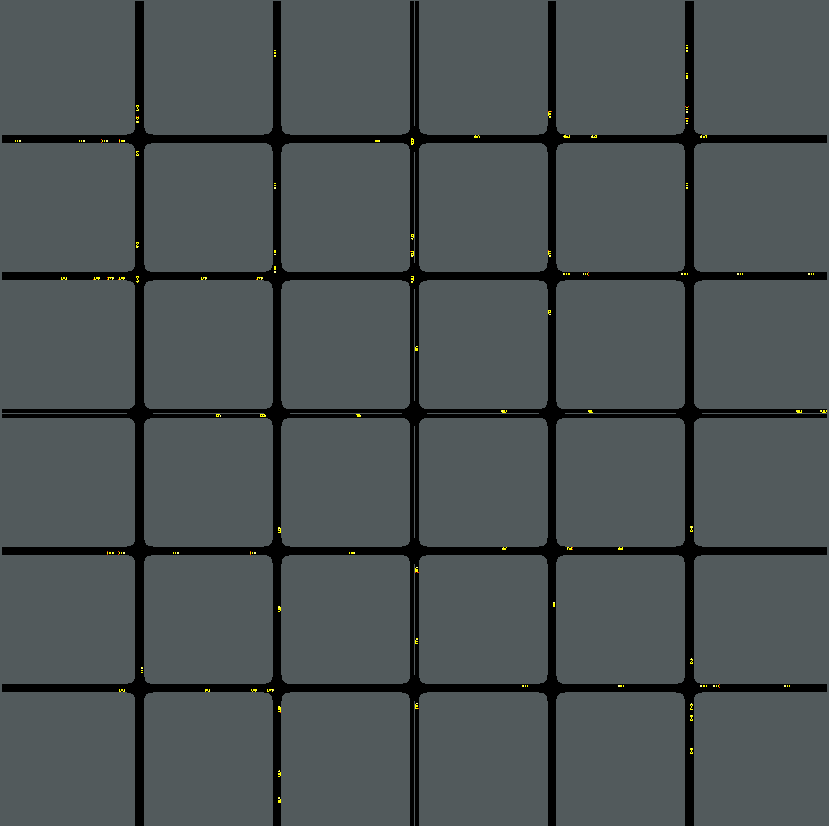}
  \end{subfigure}
  \hfill
  \begin{subfigure}{0.45\textwidth}
    \includegraphics[width=0.85\textwidth]{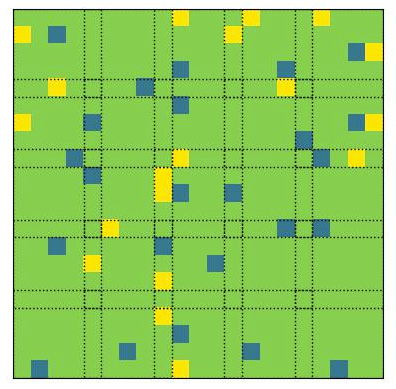}
  \end{subfigure}
\caption{A screenshot of the global simulators for the 25-agent variants of the traffic control (left) and warehouse (right) environments} 
\label{fig:global_sim}
\end{figure}

\begin{figure}[h!]
\centering
  \begin{subfigure}[b]{0.48\textwidth}
   \hspace{30mm}
    \includegraphics[width=0.49\textwidth]{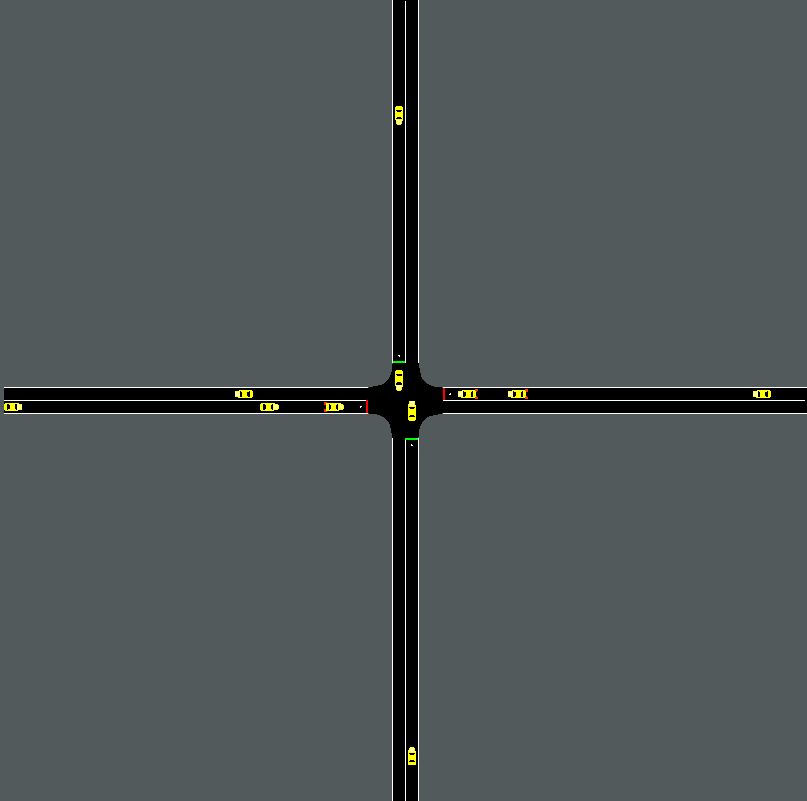}
    \vspace{2pt}
  \end{subfigure}
  \hfill
  \begin{subfigure}[b]{0.48\textwidth}
    \includegraphics[width=0.52\textwidth]{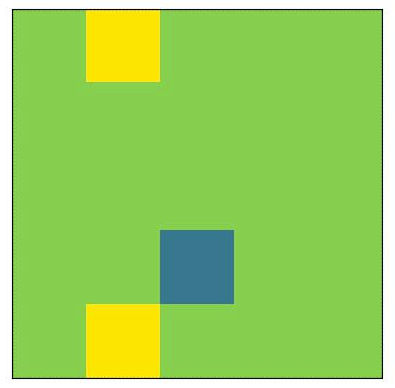}
  \end{subfigure}
\caption{A screenshot of the local simulators for the  traffic (left) and warehouse (right) environments.  Since all local regions are the same in the two environments, we use the same LS for all of them.} 
\label{fig:local_sim}
\vspace{-10pt}
\end{figure}
\section{Runtimes}\label{ap:runtimes}
The two tables below show a breakdown of the runtimes for the two environments and the three simulators. These were measured on a machine with 128 CPUs of the type AMD EPYC 7452 32-Core Processor. We used this machine for all our measurements because it is the only one in our computer cluster that can fit DIALS when training on the 100-agent variants of the environments. However, the experiments that required less than 100 CPUs were actually run on different machines.
\begin{table}[H]
\caption{Runtimes for the traffic control environment}
\resizebox{\textwidth}{!}{
\begin{tabular}{@{}lcccccccccccc@{}}
\toprule
                                      & \multicolumn{4}{c}{Agents training (h)}             & \multicolumn{4}{c}{\begin{tabular}[c]{@{}c@{}}Data collection +\\ influence training (h)\end{tabular}} & \multicolumn{4}{c}{Total (h)}  \\ \midrule
\multicolumn{1}{l|}{Number of agents} & 2    & 25    & 49     & \multicolumn{1}{c|}{100}    & 2                  & 25                 & 49                 & \multicolumn{1}{c|}{100}                & 2    & 25    & 49     & 100    \\ \midrule
\multicolumn{1}{l|}{GS}               & 7.24 & 46.96 & 105.41 & \multicolumn{1}{c|}{261.06} & -                  & -                  & -                  & \multicolumn{1}{c|}{-}                  & 7.24 & 46.96 & 105.41 & 261.06 \\
\multicolumn{1}{l|}{DIALS F=100K}     & 1.48 & 1.93  & 2.70   & \multicolumn{1}{c|}{3.70}   & 0.66               & 3.74               & 8.60               & \multicolumn{1}{c|}{22.38}              & 2.14 & 5.67  & 11.30  & 26.08  \\
\multicolumn{1}{l|}{DIALS F=500K}     & 1.48 & 1.93  & 2.70   & \multicolumn{1}{c|}{3.70}   & 0.13               & 0.75               & 1.72               & \multicolumn{1}{c|}{4.48}               & 1.61 & 2.68  & 4.42   & 8.18   \\
\multicolumn{1}{l|}{DIALS F=1M}       & 1.48 & 1.93  & 2.70   & \multicolumn{1}{c|}{3.70}   & 0.07               & 0.37               & 0.86               & \multicolumn{1}{c|}{2.24}               & 1.55 & 2.30  & 3.56   & 5.94   \\
\multicolumn{1}{l|}{DIALS F=4M}       & 1.48 & 1.93  & 2.70   & \multicolumn{1}{c|}{3.70}   & 0.02               & 0.09               & 0.21               & \multicolumn{1}{c|}{0.56}               & 1.50 & 2.02  & 2.91   & 4.26   \\
\multicolumn{1}{l|}{untrained-DIALS}  & 1.48 & 1.93  & 2.70   & \multicolumn{1}{c|}{3.70}   & -                  & -                  & -                  & \multicolumn{1}{c|}{-}                  & 1.48 & 1.93  & 2.70   & 3.70   \\ \bottomrule
\end{tabular}}
\end{table}

\begin{table}[H]
\caption{Runtimes for the warehouse environment}
\resizebox{\textwidth}{!}{
\begin{tabular}{@{}lcccccccccccc@{}}
\toprule
                                      & \multicolumn{4}{c}{Agents training (h)}              & \multicolumn{4}{c}{\begin{tabular}[c]{@{}c@{}}Data collection +\\ influence training (h)\end{tabular}} & \multicolumn{4}{c}{Total (h)}   \\ \midrule
\multicolumn{1}{l|}{Number of agents} & 2     & 25    & 49     & \multicolumn{1}{c|}{100}    & 2                  & 25                 & 49                 & \multicolumn{1}{c|}{100}                & 2     & 25    & 49     & 100    \\ \midrule
\multicolumn{1}{l|}{GS}               & 14.84 & 97.04 & 208.18 & \multicolumn{1}{c|}{468.46} & -                  & -                  & -                  & \multicolumn{1}{c|}{-}                  & 14.84 & 97.04 & 208.18 & 468.46 \\
\multicolumn{1}{l|}{DIALS F=100K}     & 2.13  & 2.56  & 3.19   & \multicolumn{1}{c|}{5.55}   & 1.32               & 7.11               & 15.19              & \multicolumn{1}{c|}{45.45}              & 4.45  & 9.67  & 18.38  & 51.00  \\
\multicolumn{1}{l|}{DIALS F=500K}     & 2.13  & 2.56  & 3.19   & \multicolumn{1}{c|}{5.55}   & 0.26               & 1.42               & 3.04               & \multicolumn{1}{c|}{9.09}               & 2.39  & 3.98  & 6.23   & 14.64  \\
\multicolumn{1}{l|}{DIALS F=1M}       & 2.13  & 2.56  & 3.19   & \multicolumn{1}{c|}{5.55}   & 0.13               & 0.71               & 1.52               & \multicolumn{1}{c|}{4.54}               & 2.26  & 3.27  & 4.71   & 10.09  \\
\multicolumn{1}{l|}{DIALS F=4M}       & 2.13  & 2.56  & 3.19   & \multicolumn{1}{c|}{5.55}   & 0.03               & 0.18               & 0.38               & \multicolumn{1}{c|}{1.13}               & 2.16  & 2.74  & 3.57   & 6.68   \\
\multicolumn{1}{l|}{untrained-DIALS}  & 2.13  & 2.56  & 3.19   & \multicolumn{1}{c|}{5.55}   & -                  & -                  & -                  & \multicolumn{1}{c|}{-}                  & 2.13  & 2.56  & 3.19   & 5.55   \\ \bottomrule
\end{tabular}}
\end{table}

\section{Memory Usage}\label{ap:memory}
The table below shows the peak memory usage of the GS and the DIALS. For the latter we provide the memory usage per process and in total. The memory needed for the GS seems to grow logarithmically with the number of agents, whereas for DIALS the memory usage per process stays relatively constant. However, the total amount of memory needed to run DIALS (aggregate of all processes) increases linearly with the number of agents and is considerably larger than that of the GS. 
\begin{table}[H]
\caption{Peak Memory Usage in Megabytes (MB)}
\resizebox{\textwidth}{!}{
\begin{tabular}{@{}llcccccccc@{}}
\toprule
\multicolumn{2}{l}{Environment}                                               & \multicolumn{4}{c}{Traffic}                             & \multicolumn{4}{c}{Warehouse}     \\ \midrule
\multicolumn{2}{l|}{Number of agents}                                          & 4     & 25     & 49      & \multicolumn{1}{c|}{100}     & 4     & 25     & 49     & 100     \\ \midrule
\multicolumn{2}{l|}{GS}                                                        & 375.3 & 392.7  & 412.5   & \multicolumn{1}{c|}{457.4}   & 339.3 & 391.8  & 469.6  & 607.4   \\ \midrule
\multicolumn{1}{l|}{\multirow{2}{*}{DIALS}} & \multicolumn{1}{l|}{Per process} & 219.5 & 221.0  & 225.8   & \multicolumn{1}{c|}{228.7}   & 195.6 & 201.9  & 203.7  & 207.5   \\
\multicolumn{1}{l|}{}                       & \multicolumn{1}{l|}{Total}       & 878.0 & 5525.0 & 11064.2 & \multicolumn{1}{c|}{22870.0} & 782.4 & 5047.5 & 9981.3 & 20750.0 \\ \bottomrule
\end{tabular}}
\label{tab:memory}
\end{table}
\section{Hyperparameters}\label{ap:hyper}

The hyperparameters used for the AIPs are reported in Table \ref{tab:AIPs}. Since feeding past local states did not seem to improve the performance of the AIPs in the traffic environment we modeled them with FNNs. In contrast, adding the past ALSHs does decrease the CE loss in the warehouse environment, and thus we used GRUs \citep{Cho2014Learning} instead. The size of the networks was chosen as a compromise between low CE loss and computational efficiency. On the one hand, we need accurate AIPs to properly capture the influence distributions. On the other, we also want them to be small enough such that we can make fast predictions. The hyperparameter named seq. length determines the number
of timesteps the GRU is backpropagated. This was chosen to be equal to the horizon such that episodes did not have to be truncated. The rest of the hyperparameters in Table \ref{tab:AIPs}, which refer to the training setup for the AIPs, were manually tuned.
\begin{table}[H]
\centering
\caption{Hyperparameters for approximate influence predictors (AIPs).}
\resizebox{\textwidth}{!}{
\begin{tabular}{@{}lllllllll@{}}
\toprule
          & Architecture & Num. layers & Num. neurons    & Seq. length & Learning rate      & Dataset size     & Batch size & Num. epochs \\ \midrule
Traffic   & FNN          & $2$         & $128$ and $128$ & -               & $1\mathrm{e}{-4}$  & $1\mathrm{e}{4}$ & $128$      & $100$       \\ \midrule
Warehouse & GRU          & $2$         & $64$ and $64$   & $100$           & $1\mathrm{e}{-4}$ & $1\mathrm{e}{4}$ & $32$       & $300$       \\ \bottomrule
\end{tabular}}
\label{tab:AIPs}
\end{table}

The hyperparmeters used for the policy networks are given in Table \ref{tab:PolicyNetworks}. We chose again GRUs for the warehouse environment and FNNs for the traffic domain, since feeding the previous AOHs did not seem to improve the agents' performance in the latter. The network size and the sequence length parameter for the GRUs were manually tuned on the smallest scenarios with 4 agents.
\begin{table}[H]
\caption{Hyperparmeters for policy networks.}
\label{tab:PolicyNetworks}
\centering
\begin{tabular}{@{}lllll@{}}
\toprule
          & Architecture & Num. layers & Num. neurons & Seq. length \\ \midrule
Traffic   & FNN          & 2           & 256 and 128    & -               \\ \midrule
Warehouse & GRU          & 2           & 256 and 128    & 8               \\ \bottomrule
\end{tabular}
\end{table}

As for the hyperparameters specific to PPO (Table \ref{tab:PPO_hyperparameters}), we used the same values reported by \citep{schulman2017proximal}, and only tuned the parameter $T$, which depends on the rewards and the episode length. $T$ determines for how many timesteps the value function is rollout before computing the value estimates.


\begin{center}
\vspace{-10pt}
  \begin{table}[h]
  \centering
  \caption{PPO hyperparameters.}
  \vspace{5pt}
\resizebox{0.5\textwidth}{!}{
  \begin{tabular}{ p{2.5cm}|p{4cm}}
 Rollout steps $T$ & 16 traffic and 8 warehouse \\
 Learning rate & 2.5e-4 \\
 Discount $\gamma$ & 0.99 \\
 GAE $\lambda$ & 0.95 \\
 Memory size & 128 \\
 Batch size & 32 \\
 Num. epoch & 3 \\
 Entropy $\beta$ & 1.0e-2 \\
 Clip $\epsilon$ & 0.1 \\
 Value coeff. $c_1$ & 1 \\
\end{tabular}
}
\vspace{-20pt}
\label{tab:PPO_hyperparameters}
\end{table}
\end{center}

\end{document}